\documentclass[11pt]{article}
\usepackage[utf8]{inputenc}
\usepackage{libertine,euler}
\usepackage{verbatim}
\usepackage[T1]{fontenc}
\usepackage{hyperref,authblk,amsmath,amsfonts,amssymb}
\usepackage{nameref,amsthm,mathtools}
\usepackage{tikz}
\usepackage{multirow}
\usetikzlibrary{shapes.geometric, arrows.meta, positioning}
\newtheorem{lemma}{Lemma} 
\newtheorem{theorem}{Theorem} 
\usepackage[margin=1in]{geometry}
\title{Interpretable Classification of Time Series Using Euler Characteristic Surfaces}

\author[1]{Salam Rabindrajit Luwang}
\author[2,a]{Sushovan Majhi}
\author[1]{Vishal Mandal}
\author[3,b]{Atish J. Mitra}
\author[1,c]{Md. Nurujjaman}
\author[1,d]{Buddha Nath Sharma}
\affil[1]{National Institute of Technology, Department of Physics, Ravangla, 737139, India}
\affil[2]{George Washington University, Data science program, Washington, DC, 20052, USA}
\affil[3]{Montana Technological University, Department of Mathematical Sciences, Butte, 59701, USA}
\affil[a]{s.majhi@gwu.edu}
\affil[b]{amitra@mtech.edu}
\affil[c]{md.nurujjaman@nitsikkim.ac.in}
\affil[d]{bnsharma09@yahoo.com}

\begin{document}

\date{}

\maketitle

\begin{abstract}
Persistent homology (PH)---the conventional method in topological data analysis---is computationally expensive, requires further vectorization of its signatures before machine learning (ML) can be applied, and captures information along only the spatial axis. For time series data, we propose Euler Characteristic Surfaces (ECS) as an alternative topological signature based on the Euler characteristic ($\chi$)---a fundamental topological invariant. The ECS provides a computationally efficient, spatiotemporal, and inherently discretized feature representation that can serve as direct input to ML models. We prove a stability theorem guaranteeing that the ECS remains stable under small perturbations of the input time series. We first demonstrate that ECS effectively captures the nontrivial topological differences between the limit cycle and the strange attractor in the R\"{o}ssler system. We then develop an ECS-based classification framework and apply it to five benchmark biomedical datasets (four ECG, one EEG) from the UCR/UEA archive. On \textit{ECG5000}, our single-feature ECS classifier achieves $98\%$ accuracy with $O(n+R\cdot T)$ complexity, compared to $62\%$ reported by a recent PH-based method. An AdaBoost extension raises accuracy to $98.6\%$, matching the best deep learning results while retaining full interpretability. Strong results are also obtained on \textit{TwoLeadECG} ($94.1\%$) and \textit{Epilepsy2} ($92.6\%$).
\end{abstract}

\section*{Introduction}
\label{sec:intro}
Topological Data Analysis (TDA) comprises a suite of mathematical, statistical, and algorithmic methods for characterizing the topological and geometric structures inherent in complex datasets~\cite{chazal2021introduction}. The most prominent technique within TDA, Persistent Homology (PH), examines the ``persistence'' of multidimensional homological features---such as connected components ($0$D), loops ($1$D), and voids ($2$D)---across multiple spatial scales~\cite{otter2017roadmap}. PH captures the evolution of these features as they appear (birth) and disappear (death), summarizing them into signatures such as persistence diagrams.

PH has been successfully deployed to analyze multidimensional time series in diverse fields, including finance~\cite{rai2024identifying}, medicine~\cite{ramos2025identifying}, and chemistry~\cite{townsend2020representation}. Recently, researchers have integrated topological tools into machine learning (ML) models for time-series classification~\cite{majumdar2020clustering} and, increasingly, to improve model explainability. This shift toward interpretability is driven by the inherent opacity of deep learning models; in biomedical contexts, ``black-box'' decisions hinder clinical trust~\cite{niu2025explainable}. To gain physician acceptance, ML models should ideally be simple, interpretable, and transparent~\cite{ichinomiya2025machine}. For instance, the authors of~\cite{ichinomiya2025machine} applied PH to recurrence plots of embedded time series to create more transparent classification pipeline.

Despite its utility, PH faces significant limitations. Its computational complexity is high, typically $O(n^{\omega })$ (where $2\le \omega <2.373$ and $n$ is the number of simplices)~\cite{roy2025euler}. This bottleneck stems from the reduction of sparse boundary matrices---a process similar to Gaussian elimination---required to track feature lifetimes~\cite{edelsbrunner2002topological}. 
Furthermore, the resulting persistence diagrams are not naturally structured as feature vectors, making them difficult to integrate directly into standard ML architectures. Consequently, there is a growing need for computationally efficient, ``machine-learning-ready'' topological summaries that can serve as alternatives to persistent homology.

The Euler Characteristic ($\chi$) is a fundamental topological invariant that summarizes the shape of a dataset into a single scalar integer. Building on this, the Euler Characteristic Surface (ECS)~\cite{roy2020understanding} serves as a multiscale spatiotemporal signature that captures the variation of $\chi$ across spatial and temporal scales. Unlike the complex matrix reductions required for PH, ECS computation simply involves counting simplices of varying dimensions. This simplicity results in a computational complexity of $O(n+R\cdot T)$, where $n$ is the number of simplices, and $R$ and $T$ represent the discrete grid sizes for scale and time parameters, respectively~\cite{roy2025euler}.

Consequently, ECS offers a computationally efficient topological descriptor that inherently integrates the temporal dimension of time-series data. As a naturally discretized feature, ECS can be used as direct input for machine learning models, bypassing the vectorization issues associated with other topological tools. The utility of ECS has been demonstrated in fluid dynamics, where it was used to characterize experimental flow patterns in drying micrometer-sized droplets and distinguish between various complex flow regimes~\cite{roy2023characterizing}, and in biomedical imaging, where it was used to identify diabetic retinopathy in retinal images with a focus on creating interpretable diagnostic tools~\cite{beltramo2021euler}. However, these applications focused on spatial data rather than temporal classification tasks. To the best of our knowledge, the use of ECS for interpretable machine learning frameworks specifically tailored to time-series data remains unexplored.

To address this gap, we develop an ECS-based classification framework for binary biomedical time-series data, applied to five benchmark datasets from the UCR/UEA Time Series Classification Archive~\cite{bagnall2018uea}: four ECG datasets (\textit{ECG5000}, \textit{TwoLeadECG}, \textit{ECG200}, \textit{ECGFiveDays}) and one EEG dataset (\textit{Epilepsy2}). The framework constructs an ECS for each time series, ranks the resulting features by their AUC-based discriminative ability on the training set, and classifies test instances using either a single-feature threshold classifier (for maximum interpretability) or an AdaBoost ensemble over the full ECS feature space (for improved accuracy).

\subsection*{Our Contribution}
The principal contributions of this work are as follows:
\begin{enumerate}
\item We introduce the $K$-window method for constructing Euler Characteristic Surfaces (ECS) from Takens-embedded time series, providing a spatiotemporal topological signature that is computationally efficient ($O(n+R\cdot T)$) and inherently discretized---a direct alternative to persistent homology ($O(n^{\omega})$). Unlike PH-based pipelines, our approach extracts topological features directly from the embedded point cloud, avoiding the need for recurrence plots, persistence image vectorization, or dimensionality reduction.
\item We establish a stability result (Theorem~\ref{TemporalStabilityECS}) confirming that the ECS is robust against small perturbations in the input time series. Furthermore, we demonstrate the ability of the ECS to distinguish between periodic and chaotic dynamics within the R\"{o}ssler system, even in the presence of noise.
\item We develop a two-level classification framework: a single-feature threshold classifier for maximum interpretability, and an AdaBoost ensemble for improved accuracy. In both cases, every classification decision is traceable to a specific scale--time coordinate of the ECS, providing transparent and topologically grounded explanations.
\item We apply the classification framework on the five benchmark biomedical datasets (four ECG, one EEG). On \textit{ECG5000}, our method achieves $98.6\%$ accuracy (AUC $0.999$), substantially outperforming the $62\%$ accuracy (AUC $0.90$) reported by a recent PH-based approach on the same dataset~\cite{ichinomiya2025machine}. Strong performance is also observed on \textit{TwoLeadECG} ($94.1\%$) and \textit{Epilepsy2} ($92.6\%$).
\end{enumerate}

\subsection*{Outline}
The rest of the paper is organized as follows. Section~\nameref{sec:method} describes the mathematical preliminaries, including the novel $K$-window method for ECS construction. Section~\nameref{sec:chaos} demonstrates ECS on the R\"{o}ssler system. Section~\nameref{sec:approach} presents the classification framework. Section~\nameref{sec:results} reports results and discussion, and Section~\nameref{sec:conclusion} concludes the work.

\section*{Methods}
\label{sec:method} 
We utilize Euler Characteristic Surfaces (ECS) to extract spatiotemporal topological signatures from the dataset. The ECS serves as a high-dimensional feature vector for time series classification. This section defines the mathematical preliminaries, supplemented by visual representations.

\subsection*{Simplicial Complexes}
\label{sc}
To extract topological information, we model the dataset as a metric space, such as the Euclidean metric, with the data embedded in $\mathbb{R}^d$. Simplicial complexes are then built over the data to uncover the underlying topology and geometry~\cite{chazal2021introduction}. A \emph{simplicial complex} is a topological space composed of multidimensional building blocks or simplices---such as vertices ($0$D), edges ($1$D), triangles ($2$D), tetrahedra ($3$D), and higher-dimensional analogs---that capture higher-order connectivity of data points. For the ease of visualization and exposition, we present the geometric definition below~\cite{dey2022computational}.
Formally, a $k$-dimensional \emph{simplex} (or simply $k$-simplex) $\sigma$ in $\mathbb R^d$ is the convex hull of any $(k+1)$ points $\{a_0,a_1,\ldots,a_k\}\subset\mathbb{R}^d$ in general position. The simplex spanned by a subset of $\sigma$ is called a \emph{face} of $\sigma$.
A \emph{geometric simplicial complex} $\mathcal{K}$ in $\mathbb{R}^d$, also known as a \emph{triangulation}, is a collection of finitely many simplices in $\mathbb{R}^d$ that satisfies the following two conditions:
\begin{enumerate}
\item $\mathcal{K}$ contains every face of each simplex in $\mathcal{K}$.
\item For any two simplices $\sigma, \tau \in \mathcal{K}$, their intersection $\sigma \cap \tau$ is either empty or a face of both $\sigma$ and $\tau$.
\end{enumerate}
The maximum dimension of any simplex in $\mathcal{K}$ is called the \emph{dimension} $k$ of $\mathcal{K}$, and we also refer to it as a \emph{simplicial $k$-complex}.
Some of the commonly constructed simplicial complexes on point clouds include Vietoris--Rips (or simply Rips complexes) and \v{C}ech complexes. In our study, we use a variation of the \v{C}ech complex known as the Alpha complex~\cite{edelsbrunner2002topological}. 

\subsection*{Nerve, Voronoi Cells, and Alpha Complexes}
The understanding of \emph{nerves} and \emph{Voronoi cells} is a prerequisite for the understanding of the Alpha complex; hence, we illustrate these topics and subsequently introduce the Alpha complex. 
Given a collection of subsets $\mathcal{U} = \{U_\alpha\}_{\alpha \in \Lambda}$ of a topological space, we define the \emph{nerve} of $\mathcal{U}$ to be the simplicial complex $N(\mathcal{U})$ whose vertex set is the index set $\Lambda$, and where a subset $\{\alpha_0, \alpha_1, \ldots, \alpha_k\} \subseteq \Lambda$ spans a $k$-simplex in $N(\mathcal{U})$ if and only if
\[
U_{\alpha_0} \cap U_{\alpha_1} \cap \cdots \cap U_{\alpha_k} \neq \emptyset.
\]
Consider a dataset $ P = \{x_1, x_2, \dots, x_n\} \subset \mathbb{R}^d $, and let $ B(x_i, r) = \{ x \in \mathbb{R}^d : \|x_i - x\| \leq r \}$ be the closed ball of radius $r$, then the \emph{\v{C}ech complex} at scale $r$ is the simplicial complex:
\[C(P, r) = \left\{ \sigma \subset P : \bigcap_{x_i \in \sigma} B(x_i, r) \neq \emptyset \right\},\]
i.e., the nerve of the closed $r$-balls around the set of points $P$.
The \emph{Voronoi cell} for $x_i\in P$ is defined as the set of points $x\in \mathbb{R}^{d}$ that are at least as close to $x_i$ as to any other point of $P$. Mathematically, the Voronoi cell for $x_i$ is given by $V(x_i) = \{ x \in \mathbb{R}^d \mid d(x, x_i) \leq d(x, x_j),\ \forall x_j \in P \}$, where $d(x_i,x_j)$ represents the distance metric.
Finally, consider a finite set of points $ P = \{x_1, x_2, \ldots, x_n\} \subset \mathbb{R}^d$ in general position, and consider $ r > 0 $. The \emph{Alpha complex} at scale $ r $ is defined as the nerve of the collection $ \{V(x_i) \cap B(x_i, r)\}_{x_i \in P} $, where $ V(x_i) $ is the Voronoi cell corresponding to the point $ x_i \in P $~\cite{edelsbrunner2002topological}.

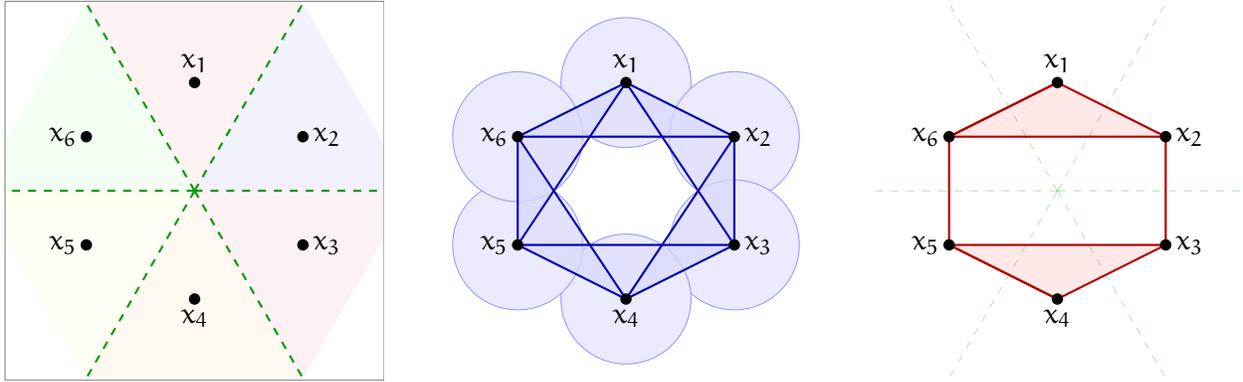
\begin{figure}[hbt]
\centering
\begin{tikzpicture}[scale=0.72]
  \clip (-3.5,-3.5) rectangle (3.5,3.5);
  \fill[blue!5] (0,0) -- (0:4) -- (60:4) -- cycle;
  \fill[red!5] (0,0) -- (60:4) -- (120:4) -- cycle;
  \fill[green!5] (0,0) -- (120:4) -- (180:4) -- cycle;
  \fill[yellow!5] (0,0) -- (180:4) -- (240:4) -- cycle;
  \fill[orange!5] (0,0) -- (240:4) -- (300:4) -- cycle;
  \fill[purple!5] (0,0) -- (300:4) -- (360:4) -- cycle;
  \foreach \angle in {0, 60, 120, 180, 240, 300} {
    \draw[green!60!black, dashed, thick] (0,0) -- (\angle:4);
  }
  \foreach \coord in {(0,2), (2,1), (2,-1), (0,-2), (-2,-1), (-2,1)} {
    \fill[black] \coord circle (3pt);
  }
  \node[above] at (0,2) {$x_1$};
  \node[right] at (2,1) {$x_2$};
  \node[right] at (2,-1) {$x_3$};
  \node[below] at (0,-2) {$x_4$};
  \node[left] at (-2,-1) {$x_5$};
  \node[left] at (-2,1) {$x_6$};
  \draw[gray, thin] (-3.5,-3.5) rectangle (3.5,3.5);
  \node[below] at (0,-3.8) {(a) Voronoi diagram};
\end{tikzpicture}
\hfill
\begin{tikzpicture}[scale=0.72]
  \clip (-3.5,-3.5) rectangle (3.5,3.5);
  \coordinate (P1) at (0,2);
  \coordinate (P2) at (2,1);
  \coordinate (P3) at (2,-1);
  \coordinate (P4) at (0,-2);
  \coordinate (P5) at (-2,-1);
  \coordinate (P6) at (-2,1);
  \foreach \coord in {(P1),(P2),(P3),(P4),(P5),(P6)} {
    \fill[blue!8] \coord circle (1.2);
    \draw[blue!40, thin] \coord circle (1.2);
  }
  \fill[blue!15, opacity=0.5] (P1) -- (P2) -- (P6) -- cycle;
  \fill[blue!15, opacity=0.5] (P2) -- (P3) -- (P1) -- cycle;
  \fill[blue!15, opacity=0.5] (P3) -- (P4) -- (P2) -- cycle;
  \fill[blue!15, opacity=0.5] (P4) -- (P5) -- (P3) -- cycle;
  \fill[blue!15, opacity=0.5] (P5) -- (P6) -- (P4) -- cycle;
  \fill[blue!15, opacity=0.5] (P6) -- (P1) -- (P5) -- cycle;
  \draw[blue!70!black, thick] (P1) -- (P2);
  \draw[blue!70!black, thick] (P2) -- (P3);
  \draw[blue!70!black, thick] (P3) -- (P4);
  \draw[blue!70!black, thick] (P4) -- (P5);
  \draw[blue!70!black, thick] (P5) -- (P6);
  \draw[blue!70!black, thick] (P6) -- (P1);
  \draw[blue!70!black, thick] (P1) -- (P3);
  \draw[blue!70!black, thick] (P2) -- (P4);
  \draw[blue!70!black, thick] (P3) -- (P5);
  \draw[blue!70!black, thick] (P4) -- (P6);
  \draw[blue!70!black, thick] (P5) -- (P1);
  \draw[blue!70!black, thick] (P6) -- (P2);
  \foreach \coord in {P1, P2, P3, P4, P5, P6} {
    \fill[black] (\coord) circle (3pt);
  }
  \node[above] at (P1) {$x_1$};
  \node[right] at (P2) {$x_2$};
  \node[right] at (P3) {$x_3$};
  \node[below] at (P4) {$x_4$};
  \node[left] at (P5) {$x_5$};
  \node[left] at (P6) {$x_6$};
  \node[below] at (0,-3.8) {(b) Nerve (\v{C}ech complex)};
\end{tikzpicture}
\hfill
\begin{tikzpicture}[scale=0.72]
  \clip (-3.5,-3.5) rectangle (3.5,3.5);
  \coordinate (P1) at (0,2);
  \coordinate (P2) at (2,1);
  \coordinate (P3) at (2,-1);
  \coordinate (P4) at (0,-2);
  \coordinate (P5) at (-2,-1);
  \coordinate (P6) at (-2,1);
  \fill[red!15, opacity=0.6] (P1) -- (P2) -- (P6) -- cycle;
  \fill[red!15, opacity=0.6] (P3) -- (P4) -- (P5) -- cycle;
  \draw[red!70!black, thick] (P1) -- (P2);
  \draw[red!70!black, thick] (P2) -- (P3);
  \draw[red!70!black, thick] (P3) -- (P4);
  \draw[red!70!black, thick] (P4) -- (P5);
  \draw[red!70!black, thick] (P5) -- (P6);
  \draw[red!70!black, thick] (P6) -- (P1);
  \draw[red!70!black, thick] (P1) -- (P6);
  \draw[red!70!black, thick] (P2) -- (P6);
  \draw[red!70!black, thick] (P3) -- (P5);
  \draw[red!70!black, thick] (P4) -- (P5);
  \foreach \coord in {P1, P2, P3, P4, P5, P6} {
    \fill[black] (\coord) circle (3pt);
  }
  \node[above] at (P1) {$x_1$};
  \node[right] at (P2) {$x_2$};
  \node[right] at (P3) {$x_3$};
  \node[below] at (P4) {$x_4$};
  \node[left] at (P5) {$x_5$};
  \node[left] at (P6) {$x_6$};
  \foreach \angle in {0, 60, 120, 180, 240, 300} {
    \draw[green!60!black, dashed, very thin, opacity=0.3] (0,0) -- (\angle:4);
  }
  \node[below] at (0,-3.8) {(c) Alpha complex};
\end{tikzpicture}
\caption{Illustration of the (a)~Voronoi diagram, (b)~Nerve (\v{C}ech complex), and (c)~Alpha complex for the hexagonal point set $S=\{(0, 2), (2, 1), (2, -1), (0, -2), (-2, -1), (-2, 1)\}$. In~(a), green dashed lines delineate the Voronoi cells. In~(b), the \v{C}ech complex at a fixed radius includes edges and triangles wherever the corresponding balls mutually intersect. In~(c), the Alpha complex restricts the \v{C}ech complex to simplices consistent with the Voronoi decomposition, yielding a sparser complex that still captures the essential topology.}
\label{fig:voronoi}
\end{figure}

We consider six points forming a regular hexagon, represented as \[S=\{(0, 2), (2, 1), (2, -1), (0, -2), (-2, -1), (-2, 1)\}.\] 
We show Voronoi cells for this dataset in Fig.~\ref{fig:voronoi} (a). Green dotted lines divide the space into six regions, each corresponding to a point. The region of a point is its Voronoi cell, which encompasses all points in the space that are closer to that point than to the other five points.

\subsection*{Euler Characteristic Surfaces}
\label{ecs}
The \emph{Euler characteristic} ($\chi$) is a topological invariant that appears throughout both pure and applied topology~\cite{roy2025euler}. 
For a (finite) simplicial complex $\mathcal{K}$, its Euler characteristic $\chi(\mathcal K)$ is defined by the alternating sum of the number of simplices in each dimension:
\begin{equation}
\chi(\mathcal K)\coloneq \sum_{i=0}^{\dim(\mathcal K)} (-1)^i~C_i(\mathcal K),
\label{eq:EC_eqn_2}
\end{equation}
where $C_i(\mathcal K)$ represents the number of $i$-dimensional simplices in $\mathcal K$. In particular, for an Alpha complex $\mathcal K_r$ constructed at scale $r$, the Euler characteristic $\chi(\mathcal K_r)$ is computed by the same alternating sum over the simplices present at that scale~\cite{carlsson2021topological}.

The Euler characteristic is invariant under homeomorphisms and is therefore unchanged by translations, rotations, affinities, projections, and continuous deformations of the underlying space~\cite{roy2020understanding}.

The \emph{Euler Characteristic Surface (ECS)} is a two-dimensional function that maps the Euler characteristic $\chi$ over a domain defined by the \emph{scale} ($r$) and \emph{time} ($k$) axes. We use the notation $k$, as opposed to the more traditional $t$ for time, as this \emph{time-like} axis will be derived from a traditional temporal axis; see Section~\nameref{subsec:takens}. The Euler characteristic of the Alpha complex $\mathcal K_{r,k}$, denoted $\chi(\mathcal K_{r,k})$, is computed at scale $r$ for the point cloud at time window $k$. 
As both parameters $r$ and $k$ vary, this defines a two-dimensional surface $\chi(\mathcal K_{r,k})\colon \mathcal R\to\mathbb{R}$ on a rectangular domain $\mathcal R\subset\mathbb R^2$. 
In practice, $\mathcal{R}$ is a discrete grid.

The similarity between two ECSs $\chi(\mathcal{K}(r, k))$ and $\chi'(\mathcal{K}'(r, k))$ defined on the same domain $\mathcal R=[0, R]\times[0, K]$ can be quantified using the \emph{$p$-Euler Metric} defined in terms of the $L^{p}(\mathcal R)$ norm as~\cite{roy2025euler}:
\begin{equation}
d_p(\chi, \chi')
= \|\chi - \chi'\|_p
= \left(
\int_{0}^{R} \int_{0}^{K}
\left| \chi(r,k) - \chi'(r,k) \right|^p
\, dk \, dr
\right)^{1/p},
\label{eq:euler_metric}
\end{equation}
where $p = 1$ or $p = 2$.
In practice, for a discrete domain $\mathcal{R}$, we substitute the above double integral with appropriate sums. 

\subsection*{The $K$-window Method for Temporal ECS Construction}
\label{subsec:takens}

We introduce the \textbf{$K$-window method} to partition a Takens-embedded point cloud into temporal blocks, enabling the construction of an ECS with an explicit time axis. This method provides a systematic way to discretize the temporal dimension of a time series into non-overlapping windows, each of which yields a separate Euler characteristic profile across filtration scales.

Given a scalar time series $\{x_t\}_{t=1}^{N}$, the embedded vectors are constructed as
\[
\mathbf{x}_t = \big(x_t,\, x_{t+\tau}, x_{t+2\tau},\ldots,x_{t+(m-1)\tau}\big), \quad t = 1, 2, \ldots, N - (m-1)\tau,
\]
where $\tau$ is the time delay and $m$ is the embedding dimension, typically chosen using the false nearest neighbors method~\cite{kennel1992determining}. 
After embedding, we obtain a point cloud $\mathcal D\subset\mathbb{R}^m$. 

The $K$-window method partitions $\mathcal D$ into $K$ equal-sized, non-overlapping subsets $\{\mathcal D_k\}_{k=1}^K$ as follows.
Fix $K\geq1$ such that $K$ divides $N-(m-1)\tau$, and set the window size $w\coloneq (N-(m-1)\tau)/K$.
The $k$-th subset is
\[
\mathcal D_k = \{\mathbf{x}_{(k-1)w+1}, \mathbf{x}_{(k-1)w+2}, \ldots, \mathbf{x}_{kw}\},\quad k=1,\ldots,K.
\]
If $N - (m-1)\tau$ is not perfectly divisible by $K$, let $w = \lfloor (N - (m-1)\tau) / K \rfloor$. We distribute the remainder by setting the first $K-1$ windows to size $w$, while the final window $\mathcal{D}_K$ absorbs all remaining points to ensure no loss of information. Specifically, $\mathcal{D}_K$ contains $(N - (m-1)\tau) - (K-1)w$ points.
For each subset $\mathcal D_k$, we construct Alpha complexes $\{\mathcal K_{r,k}\}$ over a range of resolution parameters $r$ and compute the corresponding Euler characteristic $\chi(\mathcal K_{r,k})$ at each scale. Assembling these values over all $k$ and $r$ yields the ECS.
The complete workflow is illustrated schematically in Fig.~\ref{fig:flowchart}.
\begin{figure}[hbt]
\centering
\resizebox{0.6 \textwidth}{!}{%
\begin{tikzpicture}[
    node distance=3.5cm, 
    startstop/.style = {rectangle, rounded corners, minimum width=3cm, minimum height=1.2cm, text width=2.8cm, text centered, draw=black, fill=teal!60,font=\footnotesize},
    process/.style = {rectangle, rounded corners, minimum width=3cm, minimum height=1.2cm, text width=2.8cm, text centered, draw=black, fill=teal!50, font =\footnotesize},
    arrow/.style = {thick,->,>=stealth}
]

\node (start) [startstop] {Convert a time series into a point cloud $\mathcal D$ using Takens Embedding};
\node (slide) [process, right of=start] {Choose number of windows $K$ to partition $\mathcal D$ into $K$-many equal-sized subsets $\{\mathcal D_i\}$};
\node (construct) [process, right of=slide] {Construct Alpha complex $\mathcal K_{r,k}$ for a resolution $r$ and subset $\mathcal D_k$};
\node (summarise) [process, below of=construct] {Compute Euler characteristic $\chi$ for various values of $r$ and subsets $\mathcal D_k$};

\node (ECC) [process, left of=summarise] {Construct Euler characteristic Surface $\chi(\mathcal{K}(r, k))$};
\node (l1) [startstop, left of=ECC] {Discretized input feature vector for ML models};

\draw [arrow] (start) -- (slide);
\draw [arrow] (slide) -- (construct);
\draw [arrow] (construct) -- (summarise);
\draw [arrow] (summarise) -- (ECC);
\draw [arrow] (ECC) -- (l1);

\end{tikzpicture}%
} 

\caption{Schematic workflow for constructing Euler Characteristic Surface (ECS) from a time series that provides a ``machine-learning-ready'' topological feature vector.}
\label{fig:flowchart}
\end{figure}
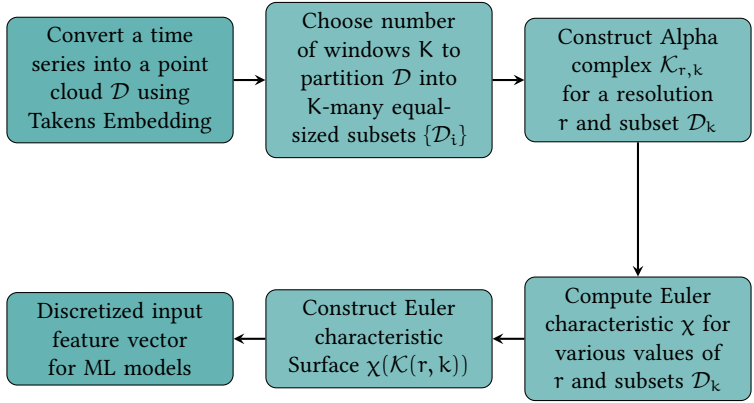

\subsection*{Stability}
\label{subsec:stability}
We prove a stability theorem guaranteeing that our pipeline remains stable under small perturbations.
Before presenting our result (Theorem~\ref{TemporalStabilityECS}), we require additional definitions and notation from algebraic topology.

Using homology of a simplicial complex $\mathcal K$, we can restate an equivalent definition of its Euler characteristic as follows~\cite{roy2025euler}:
\begin{equation}
\chi(\mathcal K)= \sum_{i=0}^{\dim(\mathcal K)} (-1)^i\ \beta_i(\mathcal K),
\label{EC_eqn_3}
\end{equation}
where $\beta_i$ is the $i$-th Betti number of the complex $\mathcal K$.

The \emph{persistence diagram}
${PD}_n(\mathcal{K})$ of a filtration of simplicial complexes $\mathcal{K}_0 \subseteq \mathcal{K}_{1} \subseteq \mathcal{K}_{2}
\subseteq \cdots \subseteq \mathcal{K}_{s}$ and homological degree $n$ is defined~\cite{otter2017roadmap} as the
multiset of points $(a_i, a_j) \in \overline{\mathbb{R}}^2$, with $i < j$,
each of which is assigned a nonzero multiplicity $\mu^n_{i,j}$, with points on the positive diagonal included with infinite multiplicity.

For $p\geq1$, the \textit{degree-p} \emph{Wasserstein Distance} between two Persistence Diagrams $PD_{n}(\mathcal{K}) \text{ and } PD_{n} (\mathcal{L})$ corresponding to the homological degree $n$ is defined as:
\[
W_p(PD_{n} (\mathcal{K}),PD_{n} (\mathcal{L}))
:= \inf_{\varphi : PD_{n} (\mathcal{K}) \to PD_{n} (\mathcal{L})}
\left(
\sum_{x \in PD_{n} (\mathcal{K})} d\bigl(x,\varphi(x)\bigr)^p
\right)^{1/p},
\]
where $\varphi$ represents a bijective mapping between $PD_{n} (\mathcal{K})$ and $PD_{n} (\mathcal{L})$. This evaluates the dissimilarity between the two persistence diagrams. 
We now prove our stability theorem, beginning with a preliminary lemma.
\begin{lemma}[Stability of Alpha Filtrations in $\mathbb{R}^3$]
\label{lem:stability}
Let $\mathcal{D}=\{x_1,\ldots,x_M\}$ be a point cloud in $\mathbb{R}^3$. 
Then, for an arbitrary small  $0< \epsilon <  \frac{1}{2} \min\limits_{1\leq i<j\leq M} \lVert x_i-x_j\rVert_2$, there is $\delta > 0 $  such that for any $\mathcal{D}'=\{x'_1,\ldots,x'_M\}$ with $\max_{1\leq i\leq M}\|x_i-x_i'\|_2\leq\delta$, we have  
\[\| \chi(\mathcal{K}(r)) - \chi(\mathcal{K}'(r)) \|_1 \le \frac{1}{4} M^4\epsilon,\]
where $\mathcal{K}(r)$ and $\mathcal{K}'(r)$ are the corresponding Alpha filtrations on $\mathcal{D}$ and $\mathcal{D}'$, respectively.
\end{lemma}

\begin{proof}
We note that for any non-degenerate simplex (triangle or tetrahedron) $T$ in $\mathbb{R}^3$ and for any $\epsilon > 0$, there is a $\delta_T > 0$ such that after perturbing the vertices by $\delta_T$, the circumradius of $T$ is changed by at most $\epsilon$ (this follows from the continuity of the circumradius function of a non-degenerate simplex).

Let $\epsilon > 0$ be less than $\frac{1}{2} \min_{1 \leq i \neq j \leq M} \|x_i - x_j\|_2$. Choose $\delta > 0$ such that for $x_i' \in \mathcal{D}'$ we have $\|x_i' - x_i\| < \min \delta_T$, where the minimum is over all non-degenerate simplices formed by points of $\mathcal{D}$. Take the filtering function on a simplex in~\cite[Theorem 2]{skraba2020wasserstein} to be the radius of the smallest enclosing sphere of the points defining the simplex, with $f$ corresponding to points in $\mathcal{D}$ and $g$ corresponding to points in $\mathcal{D}'$.

Then, by the Cellular Wasserstein theorem of Skraba and Turner, we have for $n = 0, 1, 2$ the following inequalities (for the inequalities below, note that the smallest enclosing sphere of a finite set of points in $\mathbb{R}^3$ is realized by either two of those points (diametrically opposite) on the sphere, or by three of those points on the sphere, or by four of those points on the sphere):

\[W_1(PD_0(f), PD_0(g)) \leq \binom{M}{2}\epsilon, \]

\[W_1(PD_1(f), PD_1(g)) \leq \binom{M}{2}\epsilon + \binom{M}{3}\epsilon = \binom{M+1}{3}\epsilon,  \]

\[W_1(PD_2(f), PD_2(g)) \leq \binom{M}{3}\epsilon + \binom{M}{4}\epsilon = \binom{M+1}{4}\epsilon.  \]

Finally, from Theorem 1 of \cite{roy2025euler} we have:
\begin{align*}
\|\chi(\mathcal K(r)) - \chi(\mathcal K'(r))\|_1 &\leq 2\sum_{n=0}^{2} W_1(PD_n(\mathcal K), PD_n(\mathcal K'))\\
&= 2\left[2\binom{M}{2} + 2\binom{M}{3} + \binom{M}{4}\right]\epsilon  \le \frac{1}{4} M^4 \epsilon.
\end{align*}

We sum only over $n = 0, 1, 2$ because the \v{C}ech filtration has the same persistent homology as the Alpha filtration.

\end{proof}

\noindent\textbf{Remark.} Since this paper uses the Takens embedding dimension $m=3$ for most datasets, we state the stability result here only for $\mathbb R^3$. The argument generalizes to $\mathbb R^d$ by summing over homological degrees $n=0,\ldots,d-1$ in the Wasserstein bound, yielding a bound that depends on $\binom{M+1}{d+1}$ rather than $\binom{M+1}{4}$. 

\begin{theorem}[Stability of Takens-Embedded ECS]
\label{TemporalStabilityECS}
Let $\{x_t\}_{t=1}^{N}$ be a discrete non-constant time series sampled from a smooth function $x\colon\mathbb R\to\mathbb R$, and let $0<\epsilon\leq\frac{\sqrt{3}}{2}\min\{|x_s-x_t| : 1\leq s,t\leq N, x_s\neq x_t\}$.
Then there exists $\delta>0$ such that for any discrete scalar time series $\{x'_t\}_{t=1}^{N}$ sampled from a smooth function $x'\colon\mathbb R\to\mathbb R$ with $\sup_{\mathbb R}|x(t)-x'(t)|<\delta$, we have
\[
\| \chi(\mathcal{K}(r, k)) - \chi(\mathcal{K}'(r, k)) \|_1 \le \frac{N^4\epsilon}{4K^3},
\]
where $\tau$ is the time delay and $m=3$ is the embedding dimension from the Takens theorem for both series, $\{\mathcal D_k\}_{k=1}^K$ and $\{\mathcal D'_k\}_{k=1}^K$ are the $K$-window partitions of the corresponding Takens-embedded point clouds $\mathcal D$ and $\mathcal D'$, and $\{\mathcal K_{r,k}\}$ and $\{\mathcal K'_{r,k}\}$ are the corresponding Alpha complexes for $k=1,\ldots,K$ and $r=0,\ldots,R$.
\end{theorem}
\begin{proof}
Without loss of generality, we assume that $K$ divides $N$ and set $w=N/K$.
For any $s\neq t$ with $x_s\neq x_t$, the squared Euclidean distance between the Takens-embedded points satisfies
\[
\|\mathbf{x}_s-\mathbf{x}_t\|_2^2
\geq m\left[\min\{|x_s-x_t| : 1\leq s, t\leq N, x_s\neq x_t\}\right]^2 \geq 3\cdot\frac{4}{3}\epsilon^2=4\epsilon^2.
\]
For any $1\leq k\leq K$, the point cloud $\mathcal{D}_k$ contains $w$ points, and
\[
\frac{1}{2}\min_{(k-1)w+1\leq s<t\leq kw}\|\mathbf{x}_s-\mathbf{x}_t\|_2\geq\epsilon.
\]
Lemma~\ref{lem:stability} then yields a sufficiently small $\delta>0$ such that 
\[
\| \chi(\mathcal{K}_k(r)) - \chi(\mathcal{K}'_k(r)) \|_1 \le \frac{1}{4} w^4\epsilon.
\]
Summing over all $K$ windows, we obtain
\[
\|\chi(\mathcal{K}(r, k)) - \chi(\mathcal{K}'(r, k))\|_1
=\sum_{k=1}^{K} \|\chi(\mathcal{K}_k(r)) - \chi(\mathcal{K}_k'(r))\|_1
\leq\frac{1}{4}w^4\epsilon\cdot K
=\frac{1}{4}(N/K)^4\epsilon\, K
=\frac{N^4\epsilon}{4K^3}.
\]
\end{proof}

\section*{ECS for R\"{o}ssler System}
\label{sec:chaos}

In this section, we illustrate the concept of Euler characteristic surfaces (ECS) using the well-known R\"{o}ssler system. The system comprises three coupled nonlinear differential equations that can exhibit both periodic and chaotic dynamics under continuous-time evolution~\cite{rossler1976equation}:
\begin{equation}
\label{eq:rossler}
\begin{cases}
\dot{x} &= -y - z, \\
\dot{y} &= x + a y, \\
\dot{z} &= b + z(x - c),
\end{cases}
\end{equation}
where $x$, $y$, and $z$ are state variables, and $a$, $b$, and $c$ are real-valued parameters controlling the system’s behavior. For $a = 0.20$ and $b = 0.20$, the system transitions from periodic to chaotic dynamics as the control parameter $c$ is increased.


We consider the $x$--component time series consisting of $300$ time steps from $t=0$ to $t=150$ with step size $\Delta t=0.5$, derived from Eq.~\ref{eq:rossler} for different values of $c$. Fig.~\ref{fig:rossler_ecs}(a) shows the time series for the period-one limit cycle ($c=2.3$), while Fig.~\ref{fig:rossler_ecs}(b) shows the time series for the strange attractor ($c=7.3$). We apply the Takens embedding theorem~\cite{takens2006detecting} to reconstruct the state-space trajectory in $\mathbb{R}^3$; the resulting phase portraits are shown in Figs.~\ref{fig:rossler_ecs}(c) and~\ref{fig:rossler_ecs}(d).

\begin{figure}[hbt]
    \centering
\includegraphics[width=0.62\textwidth]{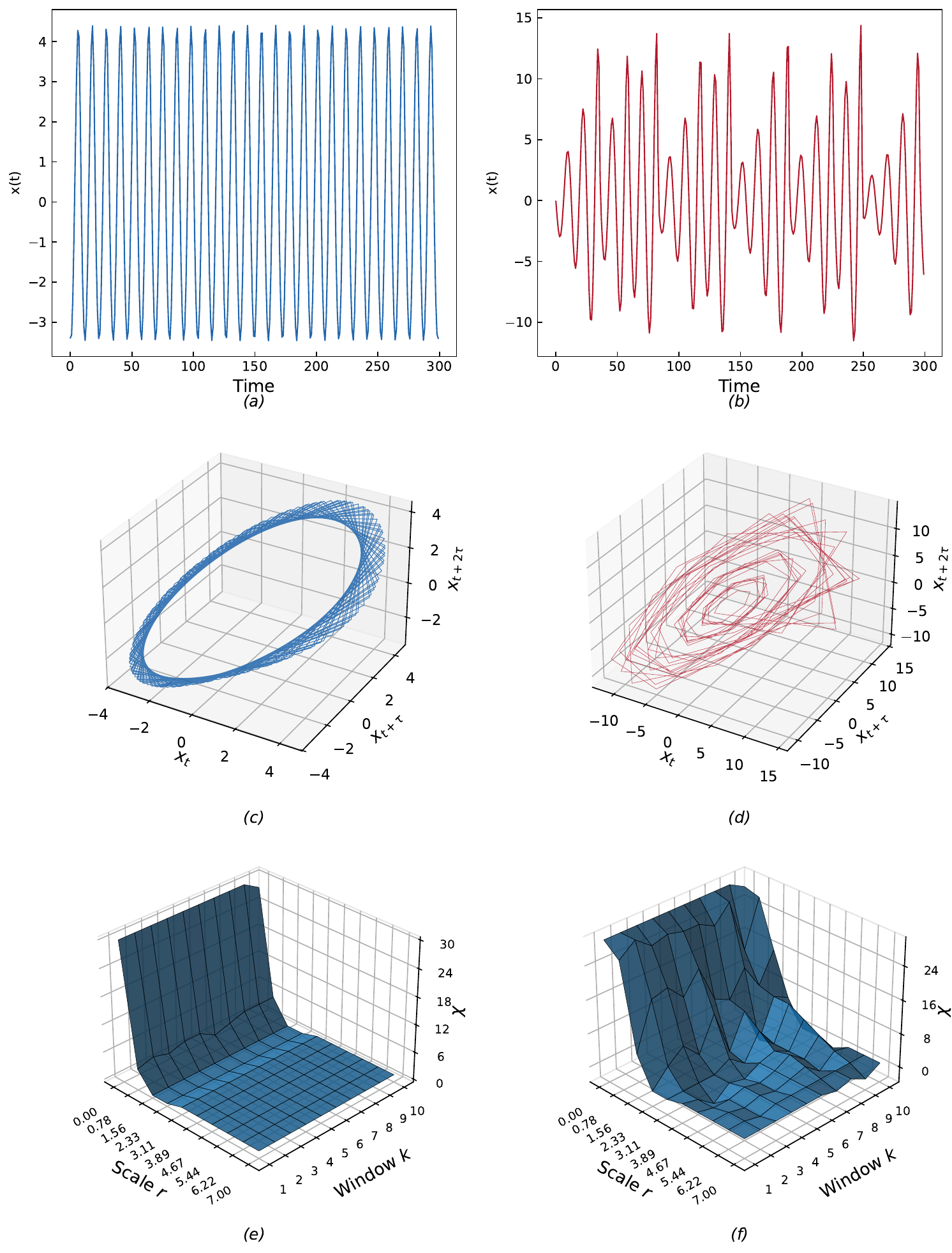}
    \caption{Time series, Takens-embedded phase portraits, and Euler Characteristic Surfaces (ECSs) of the R\"{o}ssler system for two values of the control parameter $c$. \emph{Column~1 (periodic, $c=2.3$):} (a)~$x$--component time series, (c)~phase portrait in $\mathbb{R}^3$, (e)~ECS. \emph{Column~2 (chaotic, $c=7.3$):} (b)~time series, (d)~phase portrait, (f)~ECS.}
    \label{fig:rossler_ecs}
\end{figure}

Following the scheme in Fig.~\ref{fig:flowchart} with $\tau =1$ and $m=3$, we construct the ECS for both dynamical regimes. The resulting surfaces are shown in Figs.~\ref{fig:rossler_ecs}(e) and~\ref{fig:rossler_ecs}(f), respectively. The ECS captures the variation of the Euler characteristic $\chi$ with respect to the filtration scale $r$ and the time window $k$, providing a spatiotemporal topological description of the underlying dynamics. Distinct surface structures are observed for the two parameter values, reflecting qualitative differences between the periodic and chaotic regimes.

For $c = 2.3$ (Fig.~\ref{fig:rossler_ecs}(e)), the ECS exhibits a relatively smooth and regular surface, indicating limited variation in $\chi$ across both scale and time. This behavior is consistent with the periodic nature of the dynamics. In contrast, for $c = 7.3$ (Fig.~\ref{fig:rossler_ecs}(f)), the ECS displays pronounced irregularities and stronger fluctuations. Moreover, $\chi$ takes predominantly higher values in the chaotic regime at comparable scale--time coordinates, reflecting the increased topological complexity of the strange attractor.

To quantify the distinction between the two regimes, we compute the absolute difference $|\mathrm{ECS}_{\mathrm{chaotic}} - \mathrm{ECS}_{\mathrm{periodic}}|$ across the scale--time grid (Fig.~\ref{fig:ecs_diff_heatmap}(a)). The largest divergence is concentrated at low filtration scales across all time windows, pinpointing the region where the topological signatures differ most. To further assess the stability (Theorem~\ref{TemporalStabilityECS}) of the ECS under perturbations, we add Gaussian noise to time series from both regimes and compute pairwise $L_{1}$ distances between the resulting ECSs. The noise intensity is varied from $0\%$ to $10\%$ of the amplitude of the unperturbed time series, generating $100$ noisy realizations for each regime.
\begin{figure}[hbt]
    \centering
    \includegraphics[width=15cm]{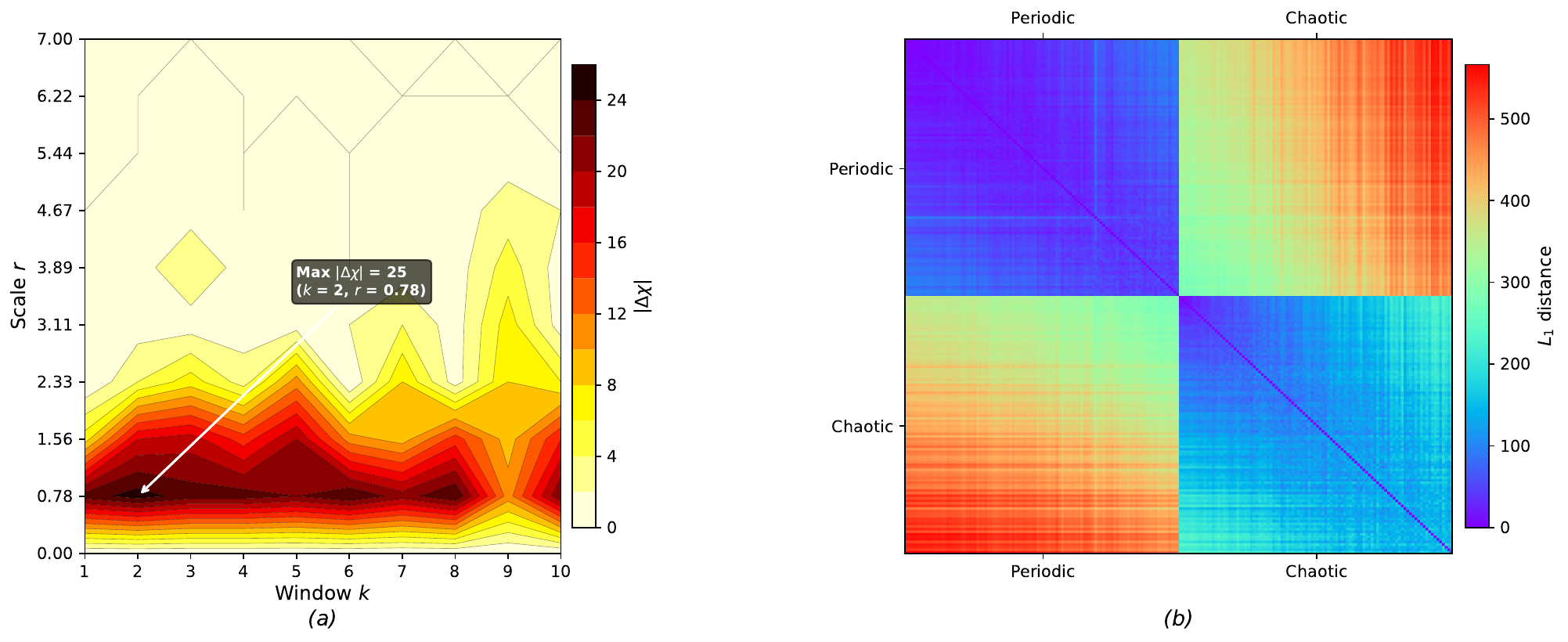}
    \caption{(a)~Absolute difference $|\mathrm{ECS}_{\mathrm{chaotic}} - \mathrm{ECS}_{\mathrm{periodic}}|$ across the scale--time grid, with the annotated arrow indicating the coordinate of maximum divergence. (b)~Heatmap of pairwise $L_{1}$ distances between ECSs constructed from perturbed time series of the periodic ($c=2.3$) and chaotic ($c=7.3$) regimes. Intra-regime distances (off- diagonal blocks) are markedly smaller than inter-regime distances (diagonal blocks).}
    \label{fig:ecs_diff_heatmap}
\end{figure}
Figure~\ref{fig:ecs_diff_heatmap}(b) shows the resulting $L_1$ distance heatmap. The intra-regime distances are markedly smaller than the inter-regime distances, even for the perturbed time series, confirming that the ECS is robust to moderate noise levels.

Overall, the ECS effectively captures the distinction in dynamical behavior of the R\"{o}ssler system as the control parameter $c$ is varied, and the $L_{1}$ metric can quantitatively differentiate between the two regimes even in the presence of noise. In the subsequent sections, we apply the ECS to classify empirical biomedical time series data.

\section*{Classification Framework}
\label{sec:approach}
We develop two classification frameworks based on Euler Characteristic Surfaces (ECSs). The first uses a single ECS feature with a threshold classifier for maximum interpretability. The second boosts multiple weak learners via Adaptive Boosting (AdaBoost) to improve accuracy while preserving interpretability.
\subsection*{Single-Feature Threshold Classifier (Decision Stump)}
\label{subsec:decision_stump}
The idea behind the single-feature classifier is simple: from the $K \times R$ ECS features, select the one that best separates the two classes and classify using a single threshold on that feature. Feature selection is performed by ranking all $K \times R$ features by their Area Under the ROC Curve (AUC) on the training set, and the optimal threshold for the top-ranked feature is determined using the Youden index. This yields a classifier that is fully interpretable---every decision traces to a single Euler characteristic value $\chi(k_i, r_j)$ at a specific scale--time coordinate of the ECS. The complete workflow is shown schematically in Fig.~\ref{fig:ecs_cls_flowchart}.
\begin{figure}[htb]
\centering
\resizebox{0.7\textwidth}{!}{%
\tikzset{
  startstop/.style = {rectangle, minimum width=3cm, minimum height=2cm, text centered, text width=3cm, draw=black, fill=orange!30, font=\bfseries},
  process/.style = {rectangle, minimum width=3cm, minimum height=2cm, text centered, text width=3cm, draw=black, fill=orange!30, font=\bfseries},
  arrow/.style = {thick,->,>=stealth}
}

\begin{tikzpicture}[node distance=4cm and 3cm]

\node (input) [startstop] {1. Input time series};

\node (embed) [process, right of=input] {2. Perform Takens embedding to generate a point cloud for each time series};

\node (tda) [process, right of=embed] {3. Compute the ECS for each point cloud as a feature vector};

\node (split) [process, right of=tda] {4. Select the best ECS feature (and corresponding threshold) based on the AUC score};

\node (select) [process, below of=split] {5. Select the best threshold for the selected feature based on Youden’s technique};

\node (apply) [process, left of=select] {6. Apply the train-derived decision rule on test data and classify signals};

\node (output) [startstop, left of=apply] {7. Evaluate performance};

\draw [arrow] (input) -- (embed);
\draw [arrow] (embed) -- (tda);
\draw [arrow] (tda) -- (split);
\draw [arrow] (split) -- (select);
\draw [arrow] (select) -- (apply);
\draw [arrow] (apply) -- (output);

\end{tikzpicture}}
\caption{The flowchart represents the ECS-based time series classification framework.}
\label{fig:ecs_cls_flowchart}
\end{figure}
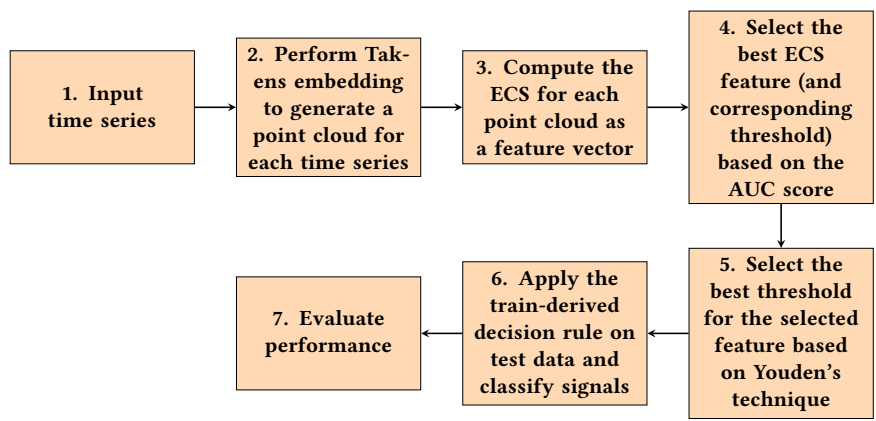

Each step in the flowchart is explained in detail below.
\begin{enumerate}
    \item A time series with the associated class label is used to initiate the process. 
    \item Each scalar time series is reconstructed to an $m$-dimensional point cloud using Takens' delay embedding theorem. We heuristically consider lag $\tau = 1$ for consistency.
    \item An ECS for each point cloud is constructed using the scheme in Fig.~\ref{fig:flowchart}. If the ECS spans $R$ spatial filtration scales and $K$ temporal windows, it yields $K \times R$ features per time series. Specifically, the Euler characteristic $\chi$ at each combination of temporal window $k_i$ and spatial scale $r_j$ constitutes one feature, numbered sequentially as $\chi(k_1,r_1), \chi(k_1,r_2), \ldots, \chi(k_K,r_R)$.
    \item For each feature $\chi(k_i,r_j)$, a receiver operating characteristic (ROC) curve is constructed from the training set. The ROC curve plots the true positive rate (TPR) against the false positive rate (FPR), defined as:
    \[
    \text{FPR} = \frac{\text{False Positives}}{\text{False Positives} + \text{True Negatives}}, \qquad
    \text{TPR} = \frac{\text{True Positives}}{\text{True Positives} + \text{False Negatives}}.
    \]
    Classification is based on a single-feature threshold classifier, also known as a ``decision stump''~\cite{viola2004robust}, defined as:
    \[
    h(\chi, p, \theta) = 
    \begin{cases} 
    1 & \text{if } p\chi < p\theta, \\
    0 & \text{otherwise,}
    \end{cases}
    \]
    where $h$ is the predicted class label, $p \in \{-1, +1\}$ is the polarity, and $\theta$ is the threshold. Features are ranked by their Area Under the ROC Curve (AUC). If the AUC is below 0.5, the polarity is reversed from $p=1$ to $p=-1$~\cite{fawcett2006introduction}. AUC-based feature ranking is motivated by its insensitivity to changes in class distribution, an attractive property for biomedical time series~\cite{fawcett2006introduction,serrano2010feature,ta2025novel,sun2017avc}. The feature with the highest AUC is selected for classifying the test set. 
    
    \item For the feature with the highest AUC, the optimum threshold is found using the Youden index~\cite{youden1950index}:
    \[
    J = \text{TPR} - \text{FPR} = \text{Sensitivity} + \text{Specificity} - 1.
    \]
    The threshold that maximizes $J$ is selected as the operating point~\cite{fluss2005estimation}.
    \item The selected feature and its Youden-optimal threshold are applied to classify the test time series. The AUC on the test set is also computed to assess classification performance.
\end{enumerate}

\subsection*{AdaBoost Ensemble Classifier}
\label{subsec:adaboost}

While the single-feature threshold classifier provides an interpretable and computationally lightweight baseline, its classification power is limited to the discriminative capacity of one ECS feature. To overcome this, we construct an ensemble of decision stumps over the full $(K \times R)$-dimensional ECS feature space using the Adaptive Boosting (AdaBoost) algorithm~\cite{schapire1999brief}. Each weak learner in the ensemble is a decision stump on a single ECS feature $\chi(k_i, r_j)$, so that the ensemble's decision remains fully interpretable: the aggregated $\alpha$-contribution of each feature can be visualized as a heatmap over the $(k, r)$ grid, directly revealing which scale--time regions drive the classification. The complete workflow is illustrated in Fig.~\ref{fig:ecs_ada_flowchart}.

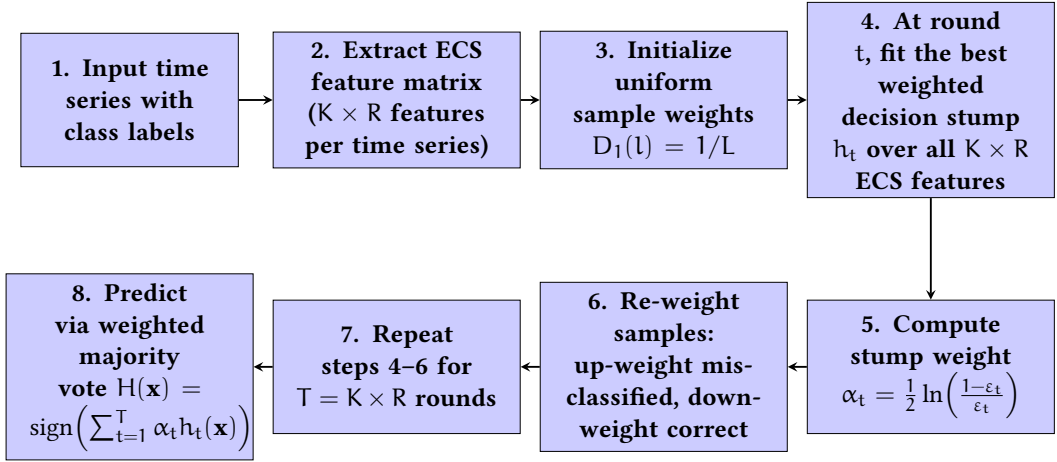
\begin{figure}[h!]
\centering
\resizebox{0.85\textwidth}{!}{%
\tikzset{
  startstop/.style = {rectangle, minimum width=3.2cm, minimum height=2cm, text centered, text width=3cm, draw=black, fill=blue!20, font=\bfseries},
  process/.style = {rectangle, minimum width=3.7cm, minimum height=2cm, text centered, text width=3cm, draw=black, fill=blue!20, font=\bfseries},
  arrow/.style = {thick,->,>=stealth}
}

\begin{tikzpicture}[node distance=4cm and 3cm]

\node (input) [startstop] {1. Input time series with class labels};
\node (ecs) [process, right of=input] {2. Extract ECS feature matrix ($K \times R$ features per time series)};
\node (init) [process, right of=ecs] {3. Initialize uniform sample weights $D_1(l) = 1/L$};
\node (stump) [process, right of=init] {4. At round $t$, fit the best weighted decision stump $h_t$ over all $K \times R$ ECS features};
\node (alpha) [process, below of=stump] {5. Compute stump weight $\alpha_t = \frac{1}{2}\ln\!\left(\frac{1-\varepsilon_t}{\varepsilon_t}\right)$};
\node (update) [process, left of=alpha] {6. Re-weight samples: up-weight misclassified, down-weight correct};
\node (converge) [process, left of=update] {7. Repeat steps 4--6 for $T = K \times R$ rounds};
\node (predict) [process, left of=converge] {8. Predict via weighted majority vote $H(\mathbf{x}) = \mathrm{sign}\!\left(\sum_{t=1}^{T}\alpha_t h_t(\mathbf{x})\right)$};

\draw [arrow] (input) -- (ecs);
\draw [arrow] (ecs) -- (init);
\draw [arrow] (init) -- (stump);
\draw [arrow] (stump) -- (alpha);
\draw [arrow] (alpha) -- (update);
\draw [arrow] (update) -- (converge);
\draw [arrow] (converge) -- (predict);

\end{tikzpicture}}
\caption{Flowchart of the ECS-based AdaBoost classification framework. 
The number of weak learners $T$ is set equal to $K \times R$, matching the 
dimensionality of the ECS feature vector.}
\label{fig:ecs_ada_flowchart}
\end{figure}

\paragraph{Standard AdaBoost procedure.}
The ECS feature matrix $\mathbf{X} \in \mathbb{R}^{L \times KR}$ is constructed for $L$ training samples using the $K$-window scheme described in Section~\nameref{subsec:decision_stump}. Each sample $l$ is initialized with equal weight $D_1(l) = 1/L$. At each boosting round $t \in \{1,\ldots,T\}$, the decision stump $h_t$ that minimizes the weighted classification error
\begin{equation}
    \varepsilon_t = \sum_{l=1}^{L} D_t(l)\; 
    \mathbf{1}\!\left[h_t\!\left(\chi_{(k_i,r_j)}^{(l)}\right) \neq c^{(l)}\right]
\end{equation}
is selected from the full set of $K \times R$ ECS features, where $c^{(l)} \in \{0,1\}$ is the true class label and $\mathbf{1}[\cdot]$ is the indicator function. If $\varepsilon_t > 0.5$, the stump polarity is reversed so that $\varepsilon_t \leq 0.5$ always holds. The stump weight is set as
\begin{equation}
    \alpha_t = \frac{1}{2} \ln\!\left(\frac{1 - \varepsilon_t}{\varepsilon_t}\right),
\end{equation}
and sample weights are updated by
\begin{equation}
    D_{t+1}(l) = 
    \frac{D_t(l)\;\exp\!\left(-\alpha_t\, c^{(l)}\, h_t\!\left(\mathbf{x}^{(l)}\right)\right)}
         {\displaystyle\sum_{l=1}^{L} D_t(l)\;\exp\!\left(-\alpha_t\, 
         c^{(l)}\, h_t\!\left(\mathbf{x}^{(l)}\right)\right)},
\end{equation}
where labels are re-encoded as $c^{(l)} \in \{-1,+1\}$ for the exponential loss. Misclassified samples receive higher weight in the next round, forcing subsequent stumps to focus on harder examples. The final prediction for a test instance $\mathbf{x}^*$ is obtained by weighted majority vote:
\begin{equation}
    \hat{c} = H(\mathbf{x}^*) = 
    \operatorname{sign}\!\left(\sum_{t=1}^{T} \alpha_t\, h_t(\mathbf{x}^*)\right).
\end{equation}

\paragraph{ECS-specific design choices.}
Three aspects of our AdaBoost implementation are tailored to the ECS feature space:
\begin{enumerate}
\item \textit{Number of rounds.} We set $T = K \times R$, so that each ECS feature---defined by a unique combination of temporal window $k_i$ and spatial scale $r_j$---has the opportunity to contribute at least one stump to the ensemble.
\item \textit{Grid size selection.} The optimal values of $K$ and $R$ (and hence $T$) are determined by 5-fold stratified cross-validation over the training set, scanning grid sizes from $2 \times 2$ to $10 \times 10$ and selecting the configuration that maximizes the mean AUC across folds.
\item \textit{Interpretability via the $\alpha$-heatmap.} Because every weak learner operates on a single ECS feature $\chi(k_i, r_j)$, the $\alpha_t$ weights accumulated over all rounds can be aggregated per feature and visualized as a two-dimensional heatmap over the $(k, r)$ grid. This provides a direct, interpretable measure of which scale--time regions of the ECS contribute most to the ensemble decision---a level of transparency rarely available in ensemble methods.
\end{enumerate}

\section*{Results and Discussion}
\label{sec:results}

We evaluate the ECS-based classification framework on five benchmark biomedical datasets from the UCR/UEA archive, progressing from a detailed analysis of the primary dataset (\textit{ECG5000}) to seizure detection (\textit{Epilepsy2}) and three additional ECG benchmarks.

\subsection*{Classification of the \emph{ECG5000} Dataset}
\label{subsec:ecg5000}

We classify the benchmark \textit{ECG5000} dataset using the ECS-based classification scheme. This dataset is sourced from the UCR time series classification archive~\cite{bagnall2018uea} and contains 20 hours of electrocardiogram (ECG) signal recorded from a single subject~\cite{doi:10.1161/01.CIR.101.23.e215}.
Each sample represents one heartbeat comprising 140 time steps. A total of 5000 samples are partitioned into five classes, with 500 samples reserved for training and 4500 for testing. Consistent with Ref.~\cite{ichinomiya2025machine}, we retain only the two most populous classes, as the remaining three are substantially underrepresented.

\subsubsection*{ECS for ECG Signals}

The phase space of each time series was reconstructed using a lag of $\tau = 1$, and the false nearest-neighbor criterion indicated an optimal embedding dimension of $m = 3$. The ECS grid size was optimized by evaluating the mean cross-validated AUC of the best single feature through 5-fold cross-validation on the training set, over grid sizes ranging from $2 \times 2$ to $10 \times 10$. The results are shown in Fig.~\ref{fig:grid_size}. The AUC rises steeply from a $3 \times 3$ grid and reaches a peak at the $10 \times 10$ configuration, which is therefore adopted for all single-feature experiments.

\begin{figure}[hbt]
    \centering
    \includegraphics[width=0.7\linewidth]{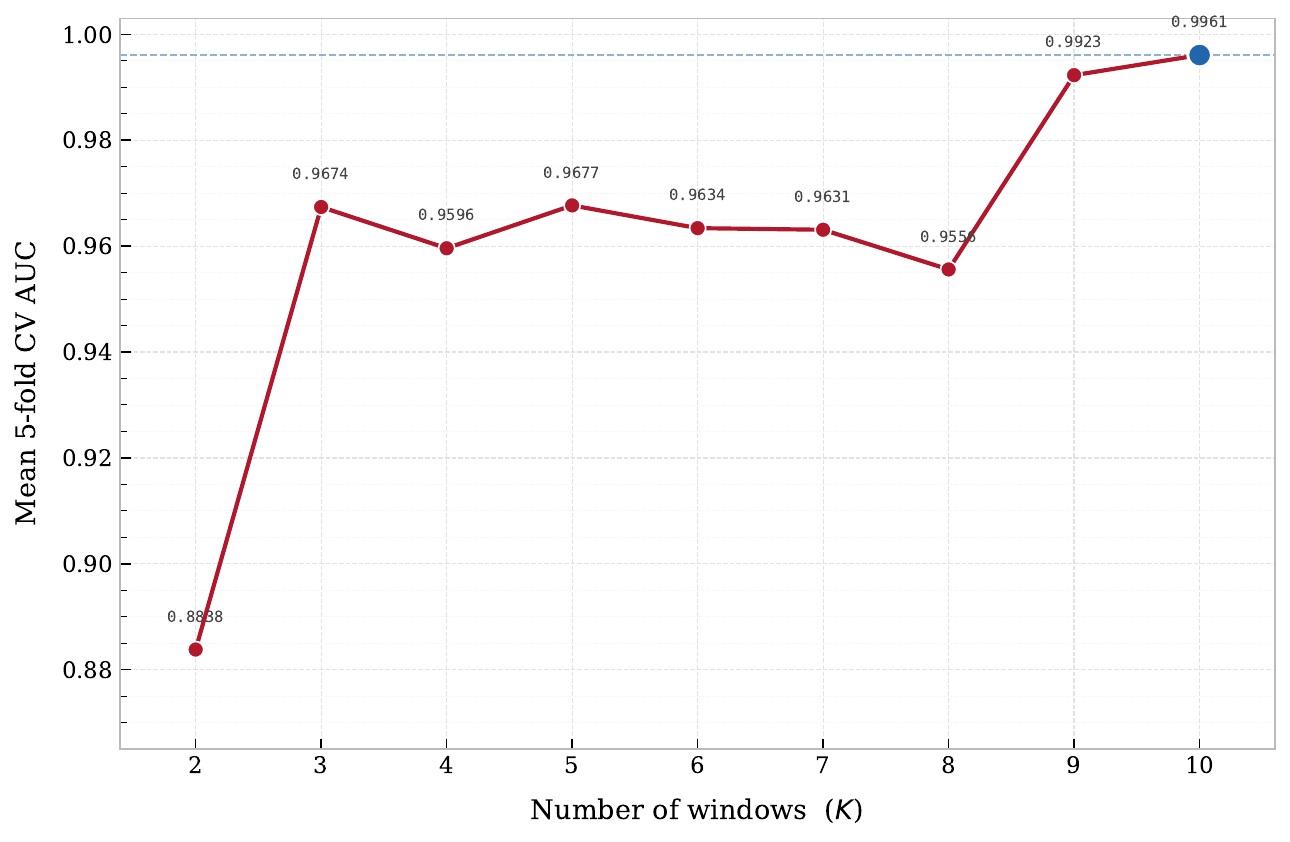}
    \caption{Mean cross-validated AUC of the best ECS feature as a function of the
    grid size ($K$) used to construct the ECS for the \textit{ECG5000} dataset.}
    \label{fig:grid_size}
\end{figure}

Figures~\ref{fig:ECS_ecg5000}(a) and (b) display representative time series from each class, while Figs.~\ref{fig:ECS_ecg5000}(c) and (d) show their respective ECSs. A clear structural difference in the Euler characteristic $\chi$ at specific scale--time coordinates is apparent between the two classes. We study differences in the Euler characteristic $|\Delta\chi|$ at the particular scale-time coordinate to understand the topological differences between the two classes. 

\begin{figure}[htb]
    \centering    \includegraphics[width=0.9\linewidth]{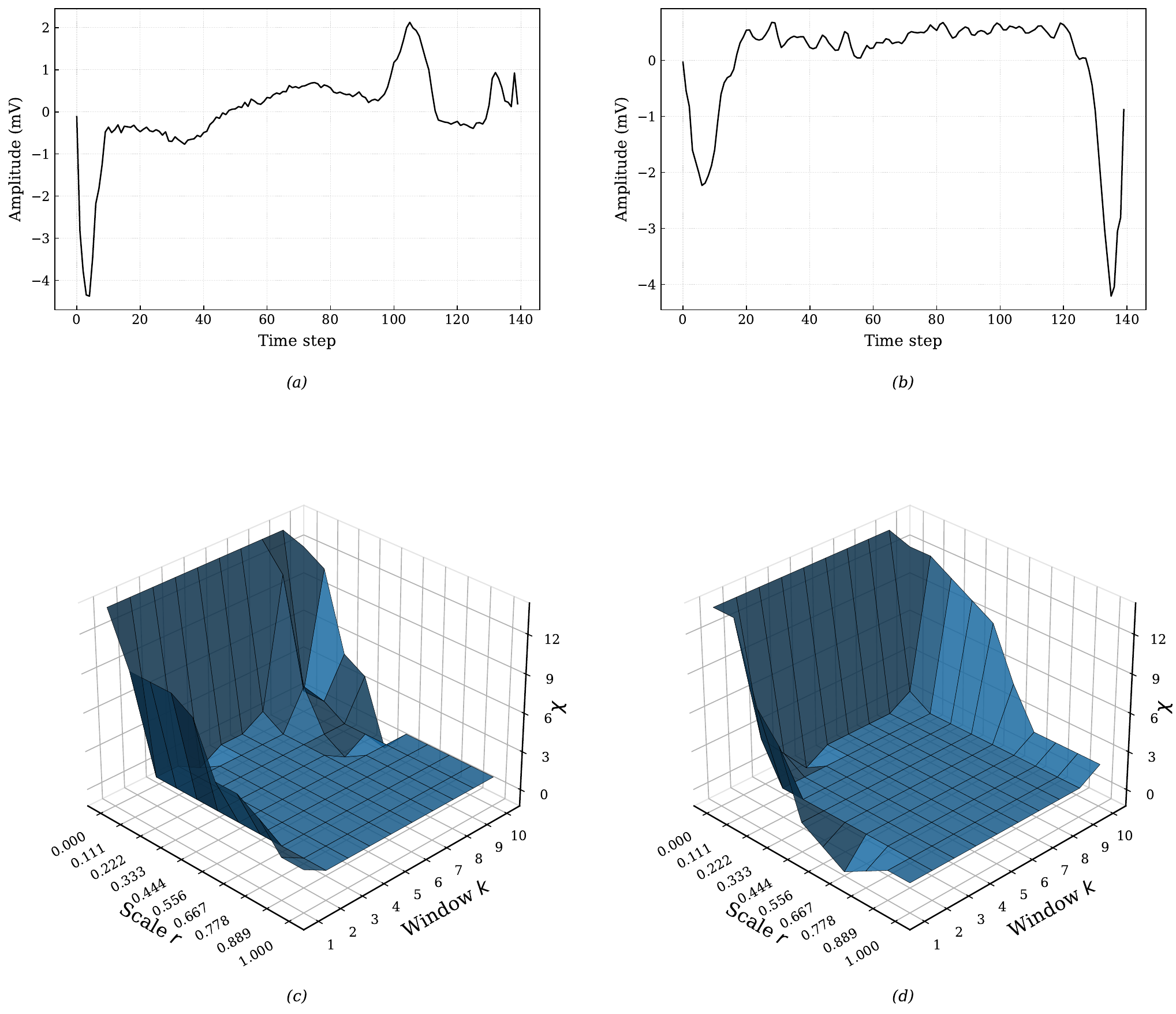}
    \caption{Panels (a) and (b) show representative ECG time series for classes 1
    and 2, respectively. Panels (c) and (d) show the corresponding ECSs.}
    \label{fig:ECS_ecg5000}
\end{figure}

Fig.~\ref{fig:ecs_ecg_heatmap} (a) represents such difference between the ECSs~\ref{fig:ECS_ecg5000}(c) and \ref{fig:ECS_ecg5000}(d). The ECSs have been interpolated to improve visual clarity. We observe a higher difference in the $\chi$ concentrated at window 8 at scale 0.11. This translates to a scale--time coordinate of approximately $(r, t) \approx (0.11, 100)$. Further, to quantify the difference in ECSs of the two classes, we compute the pairwise $L_1$ distance between all training-set ECSs using Eq.~\eqref{eq:euler_metric}. The resulting heatmap (Fig.~\ref{fig:ecs_ecg_heatmap}) confirms that intra-class distances are consistently smaller than inter-class distances. The mean intra-class $L_1$ distances for classes 1 and 2 are $5.58$ and $4.14$, respectively, while the mean inter-class distance is $8.89$, indicating that the ECS encodes topologically meaningful and class-discriminative information.

\begin{figure}[hbt]
    \centering
    \includegraphics[width=.97\linewidth]{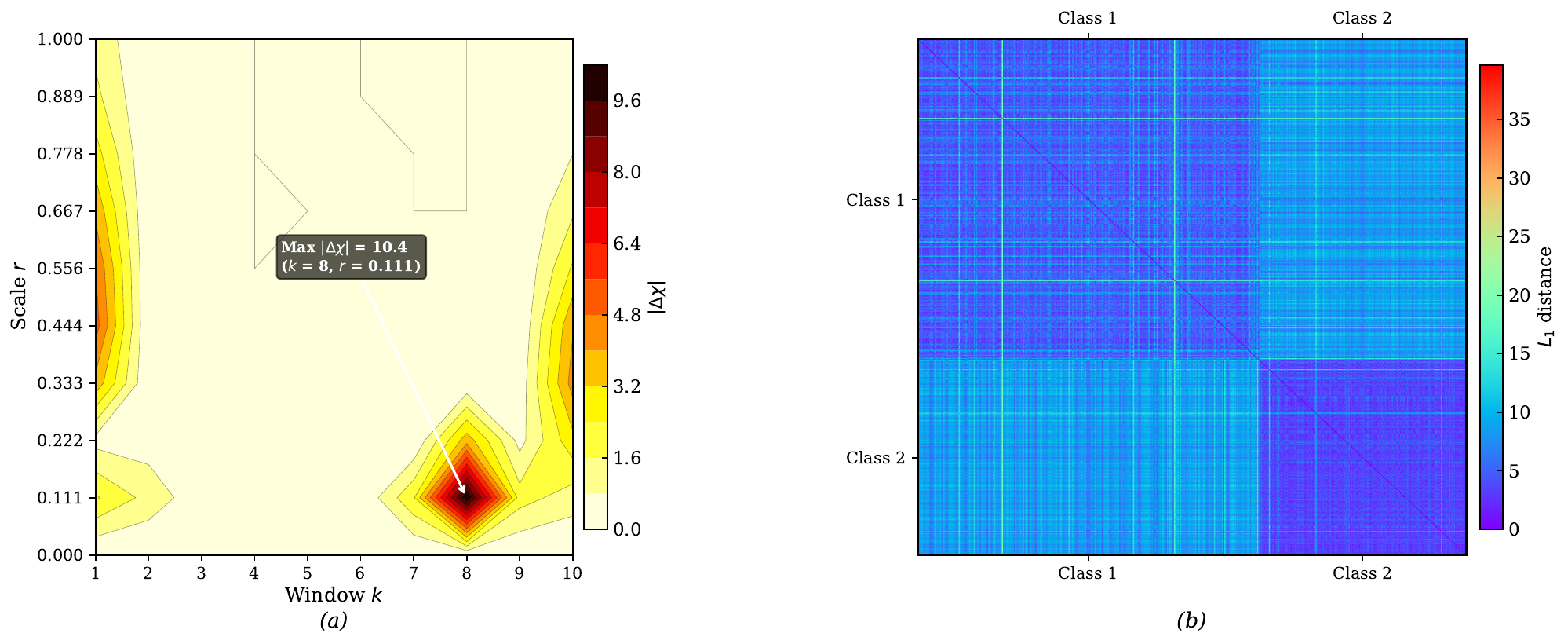}
    \caption{Topological comparison of ECS representations constructed from the \textit{ECG5000} training set. Panel~(a) shows the pointwise absolute difference $|\Delta\chi|$ between the mean ECS of each class, plotted over the scale--window grid; the annotated maximum identifies the $(k, r)$ coordinate at which the two classes are most topologically distinguishable. Panel (b) shows the pairwise $L_1$ distances between the ECSs of all training samples, visualized as a heatmap.}
    \label{fig:ecs_ecg_heatmap}
\end{figure}

\subsubsection*{Single-Feature Classification}

Ranking the $10 \times 10 = 100$ ECS features by their training-set AUC, feature $72$ emerged as the most discriminative, achieving a training AUC of $0.9962$. This feature corresponds to a scale--time coordinate at filtration scale $r \approx 0.11$ and $k=8$. This window contains the points between time step $98$ to $112$ in the Takens-embedded heartbeat signal. This is the same coordinate which showed maximum difference in Euler characteristic $|\Delta\chi|$ in Fig.~\ref{fig:ecs_ecg_heatmap} (a). The class-conditional density of this feature is shown in Fig.~\ref{fig:clsf_ecg_5000}(a). Class~1 exhibits systematically higher $\chi$ values than class~2 at this coordinate. Applying the polarity reversal ($p = -1$) and the Youden-optimal threshold $\theta = 3$, samples with $\chi_{72} > 3$ are assigned to class~1 and those below to class~2.

This single-feature decision rule has a direct physiological interpretation. The discriminative coordinate lies near the terminal portion of the heartbeat waveform, where topological differences in the Takens-embedded point cloud at scale $r \approx 0.11$ reflect structural differences in the cardiac dynamics between the two classes. The test-set ROC curve is shown in Fig.~\ref{fig:clsf_ecg_5000}(b).

\begin{figure}[hbt]
    \centering    \includegraphics[width=0.95\linewidth]{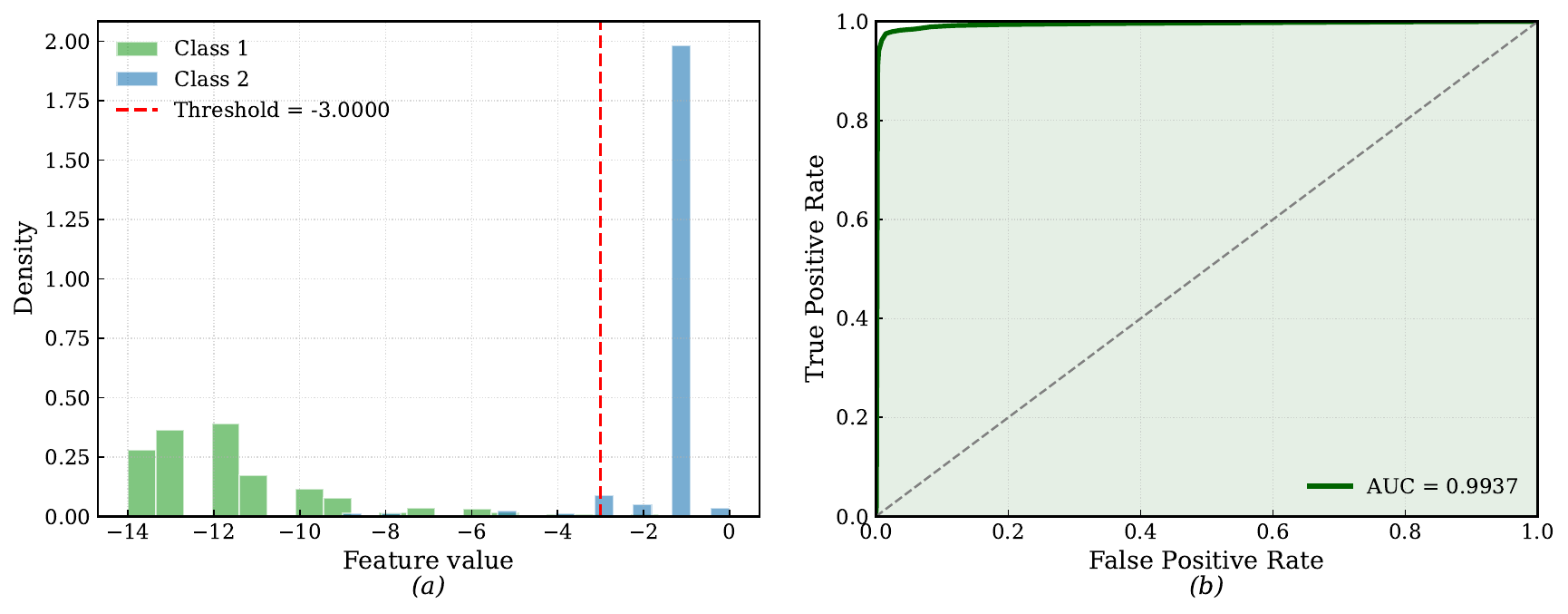}
    \caption{(a)~Training-set density of feature~72 for both classes after polarity reversal ($p=-1$), with the Youden-optimal threshold $\theta = 3$. (b)~Test-set ROC curve obtained with feature~72.}
    \label{fig:clsf_ecg_5000}
\end{figure}

\subsubsection*{AdaBoost Classification}

Cross-validation over the training set identified the $6 \times 6$ grid as optimal
for AdaBoost, yielding $T = 36$ decision stumps. The aggregated $\alpha$ contributions of all stumps over the scale--time grid are shown in Fig.~\ref{fig:adaboost_heatmap}. The $\alpha$ contribution of different features in the classification model is shown in Fig.~\ref{fig:adaboost_heatmap}.
\begin{figure}[tbh]
    \centering
    \includegraphics[width=0.7\linewidth]{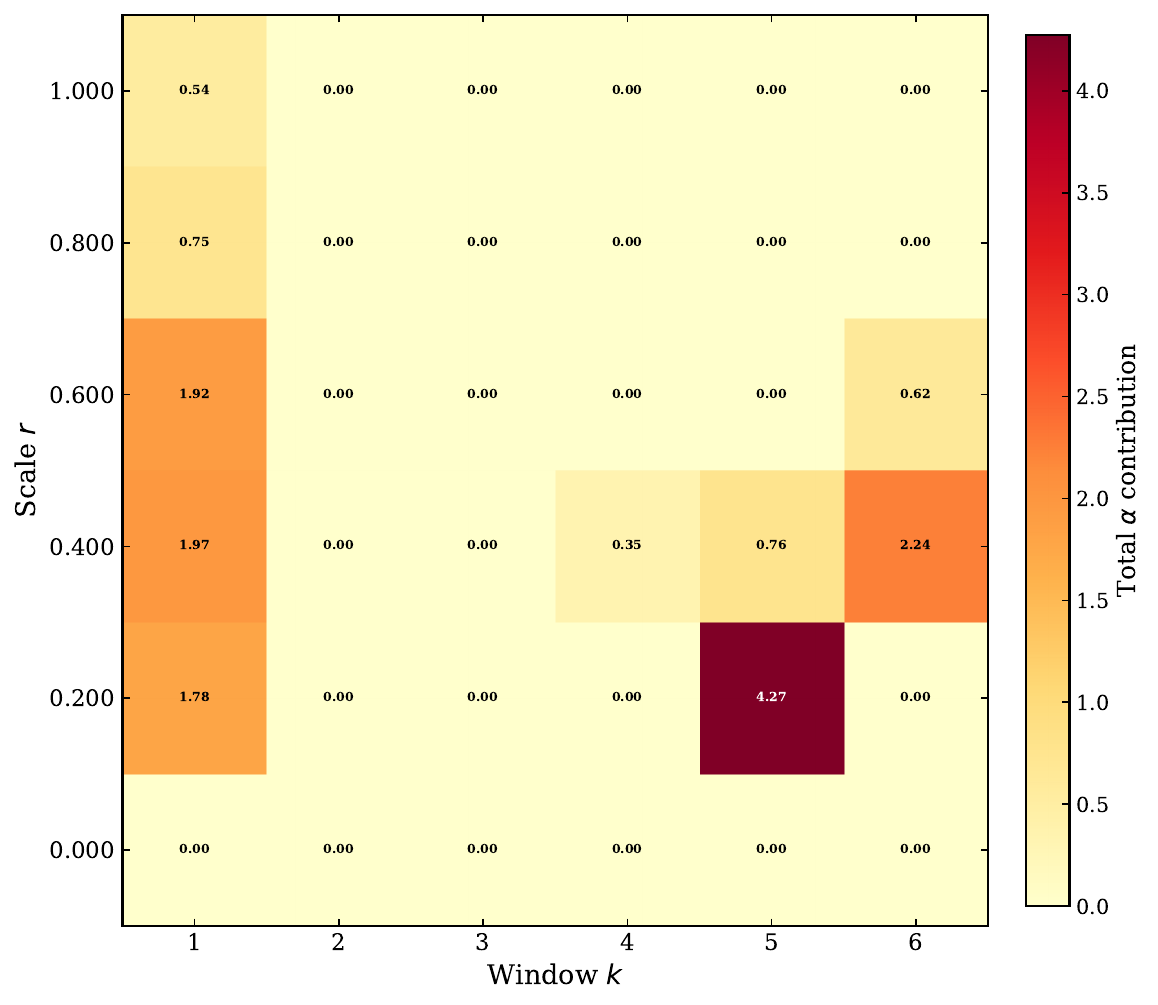}
    \caption{Heatmap of the total $\alpha$ contribution of each ECS feature in the
    AdaBoost ensemble trained on the \textit{ECG5000} dataset ($6 \times 6$ grid,
    $T = 36$ stumps).}
    \label{fig:adaboost_heatmap}
\end{figure}

The dominant feature in the $6 \times 6$ ensemble is feature 26, which corresponds to the scale--time block $(r,\,k) \approx (0.2,\,5)$. This window comprises of the point cloud from the time step $92$ to $115$ in the Takens-embedded heartbeat signal. When the AdaBoost model was computed on the $10 \times 10$ grid ECS, feature $72$ appeared to have the most $\alpha$-contribution. This consistency across grid resolutions confirms the physical robustness of the identified discriminative region; regardless of how finely the ECS is partitioned, the ensemble persistently returns to the same localized neighborhood as the most informative. The significantly higher contribution of a single feature suggests that the classification problem is governed by a highly localized topological event --- a sharp and reproducible change in the connectivity of the Takens point cloud at $r \approx 0.1-0.2$ and $time-step \approx 92-115$.

\subsubsection*{Quantitative Performance Summary}

Table~\ref{tab:ecg5000_results} summarizes the classification metrics for both frameworks on the held-out test set of $4217$ samples. The AdaBoost ensemble improves upon the single-feature baseline across every metric. The AUC rises from $0.994$ to $0.999$, accuracy from $98.0\%$ to $98.6\%$, and the F1 score from $0.974$ to $0.981$. Notably, the number of false negatives is reduced from $60$ to $37$, which is clinically significant in abnormal heartbeat detection, where missed positives carry a higher cost than false alarms.

\begin{table}[h]
\centering
\renewcommand{\arraystretch}{1.4}
\setlength{\tabcolsep}{10pt}
\begin{tabular}{lcc}
\hline\hline
\textbf{Metric} & \textbf{Single Feature} & \textbf{AdaBoost} \\
\hline
ROC-AUC              & 0.9937 & 0.9988 \\
Accuracy             & 0.9803 & 0.9855 \\
Precision            & 0.9852 & 0.9848 \\
Recall (Sensitivity) & 0.9623 & 0.9767 \\
Specificity          & 0.9912 & 0.9909 \\
F1 Score             & 0.9736 & 0.9807 \\
\hline
True Positives  & 1{,}530 & 1{,}553 \\
False Negatives &      60 &      37 \\
False Positives &      23 &      24 \\
True Negatives  & 2{,}604 & 2{,}603 \\
\hline\hline
\end{tabular}
\caption{Classification performance on the \textit{ECG5000} test set (4{,}217 samples).
The single-feature model uses a $10 \times 10$ ECS grid; the AdaBoost model uses a
$6 \times 6$ grid with $T = 36$ stumps.}
\label{tab:ecg5000_results}
\end{table}

\subsubsection*{Comparative Analysis}

Several machine learning and deep learning models have reported accuracies consistently exceeding $95\%$ on this dataset~\cite{roy2023ecg,GENG20258,XING2025111699}. However, the rationale underlying their decisions remains largely opaque, limiting clinical utility. Our framework achieves comparable accuracy while grounding every classification decision in the Euler characteristic $\chi$ --- a fundamental topological invariant --- thereby providing a transparent and physically meaningful explanation.

The closest topological precedent is the work of Ichinomiya~\cite{ichinomiya2025machine}, who constructed recurrence plots from the embedded time series and extracted topological features via persistent homology (PH), vectorized as persistence images (PIs), and further reduced by non-negative matrix factorization (NMF). This pipeline achieved an accuracy of $62\%$ and an AUC of $0.90$ with a single feature. Our method bypasses the recurrence-plot construction, PH computation, PI vectorization, and NMF reduction entirely, computing the ECS directly from the Takens point cloud. Applied to the same dataset, our single-feature scheme achieves an accuracy of $98\%$ and an AUC of $0.994$, substantially surpassing the PH-based approach. The AdaBoost extension pushes these
figures to $98.6\%$ and $0.999$, matching the best deep learning results while retaining full interpretability and incurring substantially lower computational cost.

\subsection*{Classification of the \emph{Epilepsy2} Dataset}
\label{subsec:epilepsy2}

We next classify the \textit{Epilepsy2} (recently renamed as \textit{EpilepticSeizures}) dataset from the UEA/UCR archive~\cite{bagnall2018uea}. The dataset comprises single-channel EEG recordings from $500$ subjects, each capturing $23.6$ seconds of brain activity sampled at $178$ Hz, yielding time series of $178$ time steps per sample. The recordings span five class labels corresponding to different physiological states: eyes closed, eyes open, EEG from a healthy cortical region, EEG from the tumour site, and ictal activity (seizure)~\cite{PhysRevE.64.061907}. For binary classification, the first four classes are merged into a single non-seizure class~\cite{zhang2022self}, reducing the task to discriminating seizure from non-seizure brain activity --- a clinically significant problem in automated epilepsy monitoring. The training set consists of $80$ balanced samples ($40$ seizure, $40$ non-seizure), while the test set of $11{,}420$ samples is substantially imbalanced.

\subsubsection*{ECS for EEG Signals}

The optimal embedding dimension at lag $\tau = 1$ was determined to be $m = 4$ using the false nearest neighbor criterion. The ECS grid size was optimized by 5-fold cross-validation on the training set. The best grid for both single feature and Adaboost scheme was found to be a $9 \times 9$ grid, which was therefore adopted for all experiments.

\begin{figure}[tbh]
    \centering
    \includegraphics[width=0.9\linewidth]{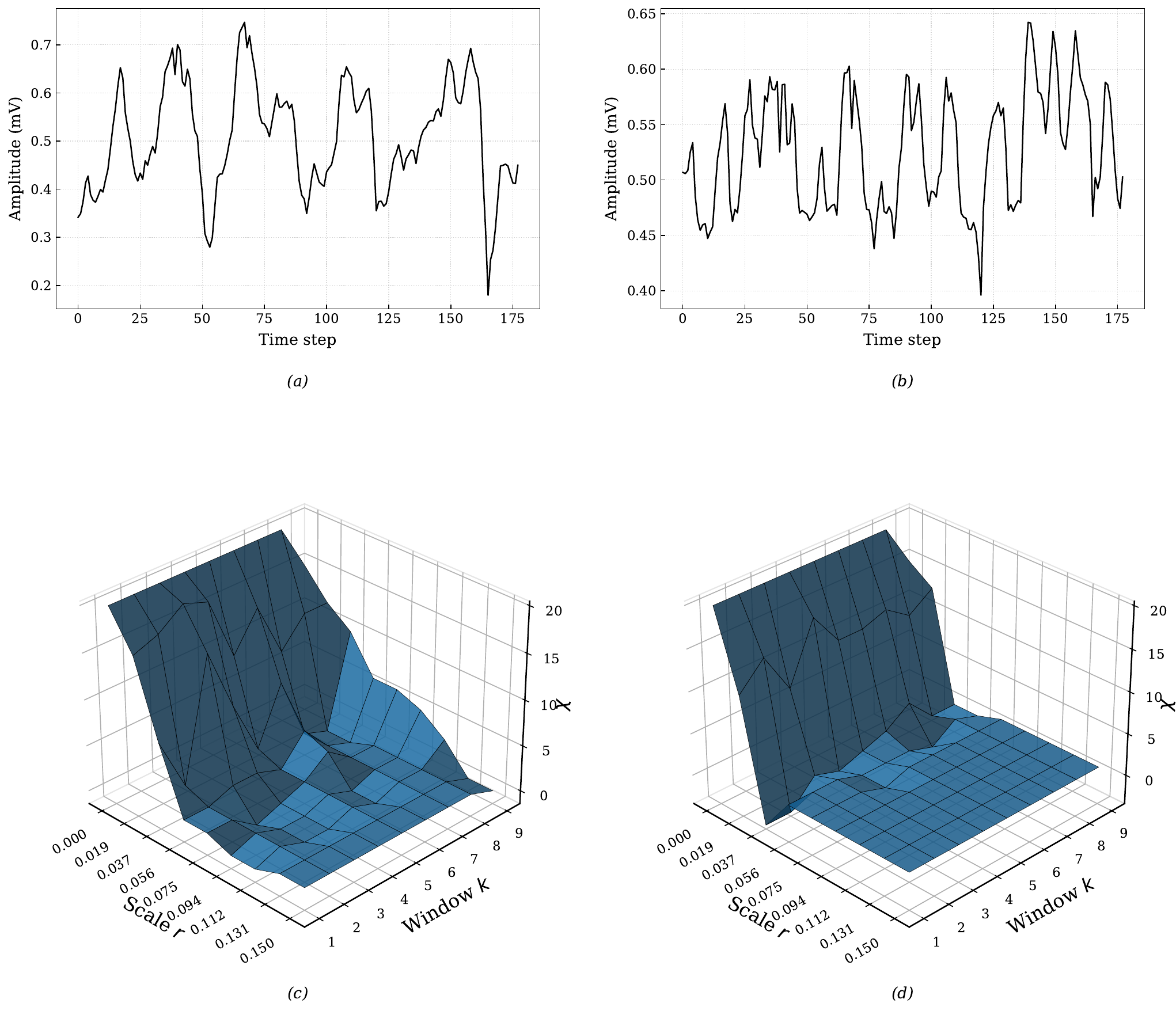}
    \caption{Panels (a) and (b) show representative EEG time series for the seizure
    and non-seizure classes, respectively. Panels (c) and (d) show the corresponding
    ECSs.}
    \label{fig:epilepsy_ecs}
\end{figure}

The ECSs for the two classes are shown in Fig.~\ref{fig:epilepsy_ecs}. The seizure ECS (Fig.~\ref{fig:epilepsy_ecs}(c)) is markedly rougher and more irregular than the smooth, slowly varying non-seizure ECS (Fig.~\ref{fig:epilepsy_ecs}(d)). This contrast reflects the well-known increase in high-frequency neural oscillations during seizure episodes, which manifests topologically as rapid fluctuations in the connectivity of the Takens point cloud across both scale and time. The difference $|\Delta\chi|$ at the particular scale-time coordinate between ECSs Fig.~\ref{fig:epilepsy_ecs}(c) and (d) is shown in  Fig.~\ref{fig:epilepsy_heatmap} (a). We observe the differences in $\chi$ to be spread across the windows in scales upto $0.075$. The pairwise $L_1$ distance heatmap (Fig.~\ref{fig:epilepsy_heatmap} (b)) confirms strong inter-class separation. Intra-class distances within the non-seizure class are very small, indicating high topological regularity, while inter-class distances are substantially larger. The higher intra-class separation for the seizure class may be caused by the inherent spatiotemporal differences in the EEG signal of different seizure episodes~\cite{https://doi.org/10.1002/hbm.25796}.

\begin{figure}[htb]
    \centering
\includegraphics[width=0.95\linewidth]{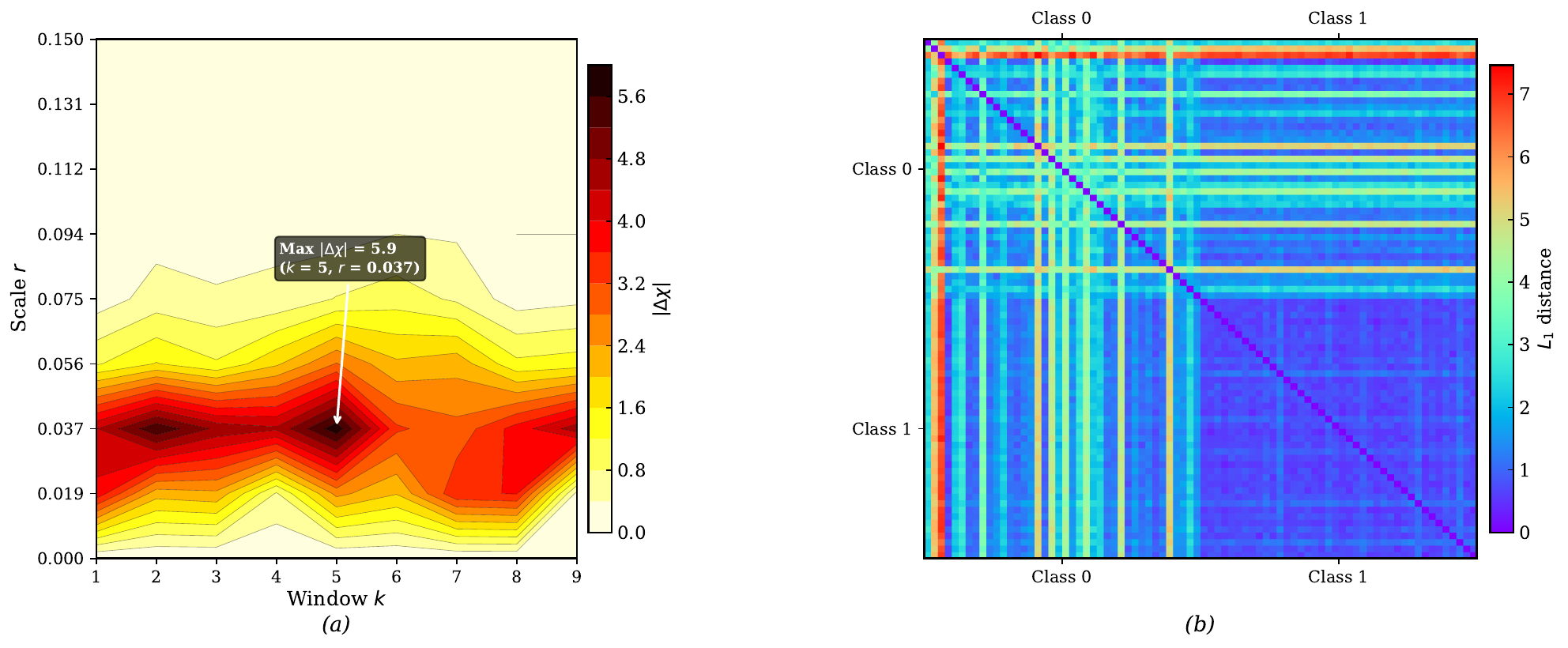}
    \caption{Topological comparison of ECS representations constructed from the \textit{Epilepsy2} training set. Panel~(a) shows the point-wise absolute difference $|\Delta\chi|$ between the mean ECS of each class, plotted over the scale--window grid; the annotated maximum identifies the $(k, r)$ coordinate at which the two classes are most topologically distinguishable. Panel~(b) shows the pairwise $L_1$ distance between the ECSs of all training samples, displayed as a heatmap.}
    \label{fig:epilepsy_heatmap}
\end{figure}

\subsubsection*{Single-Feature and AdaBoost Classification}
The classification results on the held-out test set of $11{,}420$ samples are presented in Table~\ref{tab:epilepsy2_results}. The single-feature classifier achieves an accuracy of $87.6\%$ and an AUC of $0.835$, demonstrating that a single ECS feature can provide meaningful discrimination on a large, imbalanced test set. The precision of $0.940$ is high, but the recall of $0.903$ and the false-positive count of $530$ --- representing seizure episodes misclassified as non-seizure --- reflect the difficulty of capturing all seizure episodes from a single topological coordinate.


The AdaBoost ensemble improves substantially across all metrics. The AUC rises from $0.835$ to $0.963$ --- an improvement of more than twelve percentage points --- as the ensemble draws on multiple complementary regions of the $(K \times R) = (9 \times 9) = 81$-dimensional ECS feature space. Accuracy increases to $92.6\%$, and the F1 score from $0.921$ to $0.953$. Most importantly, the false-positive count decreases from $530$ to $331$, meaning the ensemble correctly recovers $199$ additional seizure episodes that were previously misclassified as non-seizure --- the clinically critical error in seizure detection, where a missed seizure carries far greater cost than a false alarm. The false-negative count also falls from $885$ to $519$, confirming that the reduction in missed seizures is not achieved at the expense of an inflated false alarm rate --- a balanced improvement that follows directly from the AdaBoost reweighting mechanism.

\begin{table}[ht]
\centering
\renewcommand{\arraystretch}{1.4}
\setlength{\tabcolsep}{10pt}
\begin{tabular}{lcc}
\hline\hline
\textbf{Metric} & \textbf{Single Feature} & \textbf{AdaBoost} \\
\hline
ROC-AUC              & 0.8347 & 0.9628 \\
Accuracy             & 0.8761 & 0.9256 \\
Precision            & 0.9398 & 0.9631 \\
Recall (Sensitivity) & 0.9034 & 0.9433 \\
F1 Score             & 0.9212 & 0.9531 \\
\hline
True Positives  & 8{,}275 & 8{,}641 \\
False Negatives &   885   &   519   \\
False Positives &   530   &   331   \\
True Negatives  & 1{,}730 & 1{,}929 \\
\hline\hline
\end{tabular}
\caption{Classification performance on the \textit{Epilepsy2} test set (11{,}420 samples). The ECS was constructed with a $9 \times 9$ grid ($m = 4$, $\tau = 1$) and the AdaBoost ensemble comprised $T = 81$ stumps.}
\label{tab:epilepsy2_results}
\end{table}

The $\alpha$-contribution heatmap for the AdaBoost ensemble is shown in Fig.~\ref{fig:adaboost_epilepsy2}.
\begin{figure}[tbh]
    \centering
\includegraphics[width=0.82\linewidth]{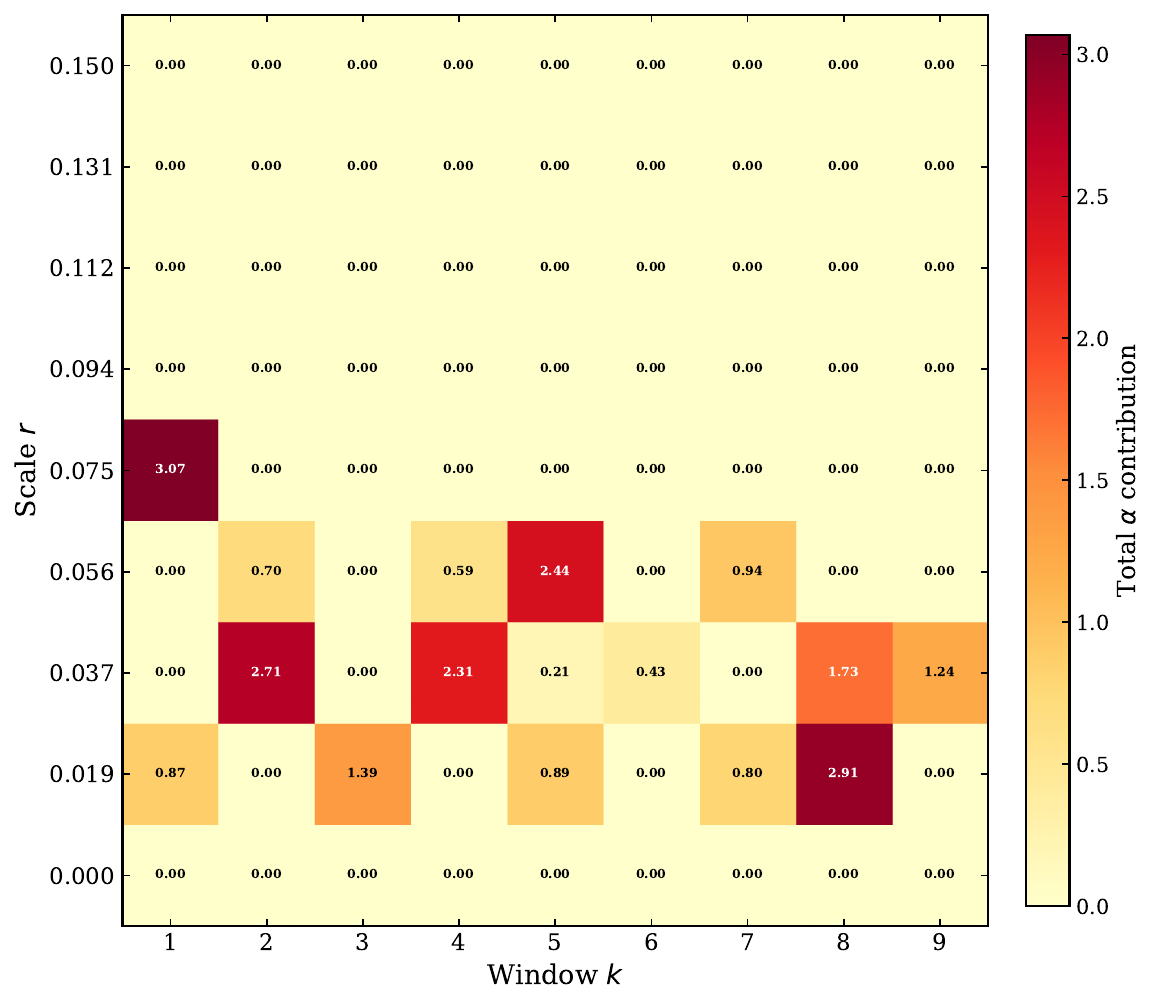}
    \caption{Heatmap of the total $\alpha$ contribution of each ECS feature in the AdaBoost ensemble trained on the \emph{Epilepsy2} dataset ($9 \times 9$ grid, $T = 81$ stumps).}
    \label{fig:adaboost_epilepsy2}
\end{figure}
Unlike the \textit{ECG5000} case, where a single feature dominates, the \textit{Epilepsy2} heatmap reveals multiple features across different temporal blocks contributing substantially to the classification. This demonstrates that when discriminatory information is spread across the spatiotemporal range of the ECS rather than concentrated at a single coordinate, the AdaBoost ensemble effectively combines these individually weak stumps into a strong classifier.

\subsection*{Additional ECG Datasets}
\label{subsec:other_datasets}

We applied the same pipeline to three further ECG datasets from the UCR archive, and the results are presented in Table~\ref{tab:all_datasets_results}.

\begin{table}[ht]
\centering
\renewcommand{\arraystretch}{1.4}
\setlength{\tabcolsep}{5pt}
\begin{tabular}{llccccccc}
\hline\hline
\textbf{Dataset} & \textbf{Method} & \textbf{AUC} & \textbf{Accuracy} &
\textbf{Precision} & \textbf{Recall} & \textbf{Specificity} & \textbf{F1}  \\
\hline
\multirow{2}{*}{\textit{TwoLeadECG}}
& Single Feature & 0.9742 & 0.9263 & 0.8894 & 0.9737 & 0.8787 & 0.9296 \\
& AdaBoost       & 0.9780 & 0.9412 & 0.9736 & 0.9070 & 0.9754 & 0.9391 \\[4pt]
\multirow{2}{*}{\textit{ECG200}}
  & Single Feature & 0.8516 & 0.7800 & 0.9565 & 0.6875 & 0.9444 & 0.8000 \\
  & AdaBoost       & 0.8811 & 0.8000 & 0.8143 & 0.8906 & 0.6389 & 0.8507 \\[4pt]
\multirow{2}{*}{\textit{ECGFiveDays}}
& Single Feature & 0.7949 & 0.7851 & 0.7431 & 0.8753 & 0.6939 & 0.8038 \\
& AdaBoost       & 0.8494 & 0.7933 & 0.7904 & 0.8014 & 0.7850 & 0.7959 \\
\hline\hline
\end{tabular}
\caption{Classification performance across three additional ECG datasets from the UCR archive. All metrics are reported on the held-out test set. The optimal ECS grid size was determined independently for each dataset by 5-fold cross-validation.}
\label{tab:all_datasets_results}
\end{table}


The \textit{TwoLeadECG} dataset yields strong performance under both frameworks. The single-feature classifier achieves an AUC of $0.974$ and an accuracy of $92.6\%$, with a notably high recall of $0.974$ at the cost of a modest precision of $0.889$, indicating that the single feature captures nearly all true positives but incurs a non-trivial false-positive rate. AdaBoost corrects this asymmetry. Precision rises sharply to $0.974$ and specificity from $0.879$ to $0.975$, while recall decreases from $0.974$ to $0.907$ --- a classic sign of boosting overcompensating for the initial false-positive surplus by shifting the effective decision boundary toward greater conservatism. The net effect is a more balanced classifier, with accuracy improving from $92.6\%$ to $94.1\%$ and the F1 score remaining stable at $0.930$ vs. $0.939$, confirming that the dataset is well-separated in the ECS feature space and that the
dominant topological discriminant is already captured by a single feature.

The \textit{ECG200} dataset is a moderately harder problem. The single-feature classifier achieves a high precision of $0.957$ and specificity of $0.944$, but at the cost of a lower recall of $0.688$, yielding an F1 score of $0.800$. This asymmetry indicates that while the selected feature rarely misidentifies a negative sample, it misses a non-trivial fraction of true positives, likely because the minority class exhibits greater topological variability than any single coordinate can capture. AdaBoost addresses this directly: by reweighting misclassified positive instances at each boosting round, the ensemble achieves a recall of $0.891$ at a precision of $0.814$ and a substantially more symmetric confusion matrix, with the AUC improving from $0.852$ to $0.881$.


The \textit{ECGFiveDays} dataset exposes the boundary conditions of our framework. The single-feature classifier achieves a recall of $0.875$ but a specificity of only $0.694$, indicating that while it captures the majority of true positives, it does so at the cost of a meaningful false-positive rate. The AUC of $0.795$ and F1 score of $0.804$ reflect a classifier that is functional but asymmetric, with the Youden-optimal threshold on the small training set placing the decision boundary at a position that does not generalize uniformly across both classes. AdaBoost substantially recovers from this imbalance; by iteratively reweighting misclassified instances, the ensemble raises specificity from $0.694$ to $0.785$ and AUC from $0.795$ to $0.849$, while recall decreases modestly from $0.875$ to $0.801$ --- again a signature of boosting correcting an initial false-positive surplus by shifting the effective decision boundary. The result is a markedly more symmetric classifier, with precision rising from $0.743$ to $0.790$ and the F1 score remaining stable at $0.804$ vs.\ $0.796$. The comparatively modest absolute performance on this dataset likely reflects the fact that the topological differences between classes are distributed across many scale--time regions of the ECS rather than concentrated at a single discriminative coordinate --- a regime in which ensemble methods hold a structural advantage over any single-feature approach.

Taken together, the results across all five datasets establish two consistent trends. First, the single-feature ECS classifier is a competitive and interpretable baseline whenever the two classes differ sharply at a localized scale--time coordinate, as is the case for \textit{ECG5000} and \textit{TwoLeadECG}. Second, the AdaBoost ensemble provides a robust improvement in every case, most critically when the single-feature baseline produces asymmetric or degenerate classification, as observed for \textit{ECG200}, \textit{ECGFiveDays}, and \textit{Epilepsy2}. The ensemble achieves this by systematically redistributing classifier attention across the full $(K \times R)$-dimensional ECS feature space, recovering topological discriminants that are individually weak but collectively powerful. This complementarity between the interpretable single-stump baseline and the robust AdaBoost ensemble constitutes the principal methodological contribution of our framework.

\section*{Conclusion}
\label{sec:conclusion}

We have presented a computationally efficient and interpretable topological framework for binary time series classification using Euler Characteristic Surfaces (ECS). The framework reconstructs the phase space of a scalar time series through Takens delay embedding and applies the $K$-window method --- introduced in this work --- to partition the embedded point cloud into temporal blocks, from which the ECS is computed as a spatiotemporal matrix tracking the Euler characteristic $\chi$ across both filtration scale and time.

Classification is performed at two levels: a single-feature threshold classifier that selects the most discriminative ECS coordinate via AUC ranking and a Youden-optimal threshold, and an AdaBoost ensemble that combines $K \times R$ decision stumps across the full ECS feature space. In both cases, every classification decision is directly traceable to a specific scale--time coordinate, providing a level of interpretability that is absent from deep learning and most machine learning approaches. We have further proved a stability theorem (Theorem~\ref{TemporalStabilityECS}) guaranteeing that the ECS remains bounded under small perturbations of the input time series, lending theoretical support to the robustness observed empirically.

Compared to persistent homology (PH)-based frameworks, the ECS offers three concrete advantages: it is computationally cheaper ($O(n + R \cdot T)$ vs.\ $O(n^{\omega})$), it is natively discretized and can serve as a direct feature vector without vectorization, and it encodes topological information along both spatial and temporal axes simultaneously. These advantages are borne out by the experiments: on \textit{ECG5000}, our framework achieves $98\%$ accuracy and an AUC of $0.937$ with a single feature, substantially outperforming a recent PH-based approach ($62\%$, AUC $0.90$)~\cite{ichinomiya2025machine}. AdaBoost improves the accuracy to $98.6\%$ and AUC to $0.999$, matching the best deep learning results while retaining full interpretability. Strong performance is also observed on \textit{TwoLeadECG} ($94.1\%$) and the EEG-based \textit{Epilepsy2} seizure detection task ($92.6\%$), demonstrating generality across signal types.

Several directions remain open for future work. The maximum scale of the ECS is chosen per dataset by inspecting the inter-point distances of the Takens-embedded point cloud. The grid parameters $K$ and $R$ are currently chosen by cross-validation. These could be optimized by a principled criterion;  perhaps grounded in the intrinsic dimensionality of the point cloud or the autocorrelation structure of the time series. The stability theorem is stated for embedding dimension $m = 3$; extending it to general $m$ is a natural theoretical goal. Extensions to multivariate time series, longer recordings, and multiclass settings (where the single-threshold framework must be replaced by a multiclass strategy) are also important next steps. Beyond biomedical signals, we anticipate that the spatiotemporal topological encoding provided by the ECS will find applications in other domains where interpretable time series classification is important, including finance, climate science, and fault detection in industrial machinery.
\bibliographystyle{unsrt}
\bibliography{sample}

\section*{Acknowledgements}
We acknowledge NIT Sikkim for allocating doctoral fellowships to B.N.S., V.M., and S.R.L.
\section*{Author contributions statement}

B.N.S. developed the computational framework and conducted the experiments. V.M. and S.R.L. contributed to data analysis. M.N. supervised the experimental work. A.J.M. formulated and proved the stability results. S.M. conceived the project and supervised the overall research direction. All authors reviewed the manuscript.

\section*{Competing interests}
The authors declare no competing interests.

\end{document}